\documentclass[12pt,english,reqno]{amsart}

\usepackage{geometry}
\usepackage{amssymb,amsthm,amsmath}
\usepackage[dvips]{graphicx}
\usepackage{subcaption}
\usepackage{color}
\usepackage{enumerate}
\usepackage{comment}
\usepackage{textcomp}
\usepackage{bbm}
\usepackage[svgnames]{xcolor}
\usepackage{tikz}
\usetikzlibrary{decorations.pathreplacing, positioning}

\usepackage[bookmarksnumbered=true]{hyperref} 
\hypersetup{
	colorlinks = true,
	linkcolor = Blue,
	anchorcolor = blue,
	citecolor = Green,
	filecolor = blue,
	urlcolor = FireBrick
}
\usepackage{orcidlink}

\newcommand{\R}{{\mathbb R}}

\newcommand{\N}{{\mathbb N}}

\newtheorem{theorem}{Theorem}
\newtheorem{lemma}[theorem]{Lemma}

\theoremstyle{definition}
\newtheorem{definition}[theorem]{Definition}

\numberwithin{equation}{section}

\begin{document}

\title[Neural operators for reaction-diffusion]{Approximation Theory of Laplacian-Based Neural Operators for Reaction-Diffusion Systems}
\date{}

\author[Furuya]{Takashi Furuya}
\address{Doshisha University, RIKEN AIP, Japan}
\email{tfuruya@mail.doshisha.ac.jp}

\author[Ozawa]{Ryo Ozawa}
\address{Tohoku University, Japan}
\email{ozawa.ryo.q6@dc.tohoku.ac.jp}

\author[Wang]{Jenn-Nan Wang\,\orcidlink{0000-0002-5651-1817}}
\address{Institute of Applied Mathematical Sciences, National Taiwan University, Taipei 106, Taiwan.}
\email{jnwang@math.ntu.edu.tw}


\begin{abstract}

Neural operators provide a framework for learning solution operators of partial differential equations (PDEs), enabling efficient surrogate modeling for complex systems. 
While universal approximation results are now well understood, quantitative approximation analysis specific to nonlinear reaction-diffusion systems remains limited. 
In this paper, we study Laplacian-based neural operators applied to the solution mapping from initial conditions to time-dependent solutions of nonlinear reaction-diffusion systems. By exploiting the Laplacian spectral representation of the diffusion semigroup, we obtain polynomial complexity bounds on the network depth, width, and spectral rank, demonstrating that Laplacian eigenfunction-based neural operators can exploit the structure of the underlying PDE while avoiding the curse of parametric complexity encountered in generic operator learning.
Numerical experiments support the theoretical findings and illustrate the benefit of Laplacian eigenfunction representations compared with Fourier-based spectral representations.

\end{abstract}

\maketitle

\begin{sloppypar}

\section{Introduction}

Reaction-diffusion equations constitute a fundamental class of partial differential equations (PDEs) that describe the interplay between local reactions and spatial diffusion. 
They arise naturally in numerous scientific disciplines, including chemical kinetics, biology, ecology, materials science, and neuroscience. 
Classical examples include the Fisher-KPP equation for population dynamics \cite{fisher1937wave,kolmogorov1937study}, the Gierer-Meinhardt system for self-organized pattern formation \cite{gierer1972theory},  the Gray-Scott model for chemical pattern formation \cite{gray1984autocatalytic, pearson1993complex}, and the FitzHugh-Nagumo system for nerve impulse propagation \cite{fitzhugh1961impulses,nagumo1962active}. 
Owing to their rich nonlinear dynamics, reaction-diffusion systems exhibit a wide variety of phenomena, including traveling waves, Turing patterns, and spatiotemporal oscillations. Consequently, they serve as important mathematical models in both applied sciences and numerical analysis.

Classical numerical methods, such as finite difference, finite element, and spectral methods, provide accurate approximations of reaction-diffusion equations. 
However, these approaches often become computationally expensive when repeated simulations are required for varying initial conditions, model parameters, or forcing terms. 
Such many-query settings frequently arise in uncertainty quantification, inverse problems, optimization, and real-time prediction. 
These challenges have motivated the development of learning-based surrogate models that approximate the underlying solution operators rather than individual solutions.

Among recent advances in scientific machine learning, neural operators have emerged as a powerful framework for learning mappings between infinite-dimensional function spaces. 
Unlike conventional neural networks, which approximate finite-dimensional input-output relationships, neural operators are designed to learn operators that map input functions directly to output functions while remaining applicable across different discretizations. 
Representative neural operator architectures include the Graph Neural Operator \cite{li2020neural}, the Fourier Neural Operator (FNO) \cite{li2020fourier}, the Wavelet Neural Operator \cite{tripura2022wavelet,gupta2021multiwavelet}, the Spherical Fourier Neural Operator \cite{bonev2023spherical}, the Laplace Neural Operator \cite{chen2023learning}, and, more recently, continuum transformer neural operators \cite{calvello2024continuum}. 
These architectures have demonstrated remarkable empirical success across a broad range of PDEs while offering mesh-independent representations, efficient inference after training, and strong generalization across families of input functions.

The expressive power of neural operators is supported by a growing body of theoretical work. 
Universal approximation theorems for classical neural networks were first established by Cybenko \cite{cybenko1989approximation} and later extended to operators between function spaces by Chen and Chen \cite{chen1995approximation}. 
Building on these foundational results, universal approximation theorems have been proved for several neural operator architectures, including DeepONets \cite{lu2021learning,lanthaler2022error}, Fourier Neural Operators \cite{kovachki2021universal}, Analytic Neural Operators \cite{lanthaler2023nonlocal}, and continuum transformer neural operators \cite{calvello2024continuum}. 
Moreover, neural operators have been shown to approximate both the forward and inverse Calder\'on maps arising in inverse conductivity problems \cite{castro2024calderon}, further demonstrating their ability to learn nonlinear solution operators.

Besides expressivity, the efficiency of neural operators is an equally important theoretical consideration. 
General neural operators satisfying only mild regularity assumptions may suffer from the curse of parametric complexity, whereby the number of trainable parameters grows exponentially with the desired approximation accuracy \cite{lanthaler2023curse}. 
Fortunately, this phenomenon can often be avoided for structured solution operators arising from well-posed PDEs. 
Quantitative approximation results have been established for Fourier Neural Operators applied to Darcy flow and the Navier-Stokes equations \cite{kovachki2021universal}, PCA-Net \cite{lanthaler2023operator}, Hamilton-Jacobi neural operators \cite{lanthaler2023curse}, nonlinear parabolic equations \cite{furuya2024quantitative}, scattering operators \cite{Furuya-Wang25}, and DeepONets for elliptic, parabolic, and hyperbolic PDEs \cite{chen2023deep,lanthaler2022error,marcati2023exponential,deng2021convergence}. 
Collectively, these results demonstrate that neural operators can efficiently approximate many PDE solution operators while maintaining favorable approximation complexity.

\subsection{Related work}

The results above suggest a broader principle: quantitative operator approximation can be substantially improved by exploiting the structure of the underlying PDE rather than treating its solution operator as a generic infinite-dimensional map.
This motivates neural operator architectures whose nonlocal components are adapted to the differential operator governing the PDE.

A prominent example is the Fourier Neural Operator (FNO) \cite{li2020fourier}, whose nonlocal layers are parameterized in the Fourier domain. 
Fourier representations are particularly natural for PDEs on periodic or rectangular domains and enable efficient implementations through the FFT.
However, on a general bounded domain with prescribed boundary conditions, the Fourier basis need not reflect the geometry of the domain or the boundary conditions of the underlying PDE.

For diffusion-dominated PDEs, a natural alternative is to exploit the spectral structure of the Laplacian itself. 
The eigenfunctions of the Laplacian are closely tied to the diffusion dynamics and depend on both the geometry of the domain and the prescribed boundary conditions. 
They therefore provide a PDE-adapted basis, rather than a generic representation of spatial functions. 
This motivates Laplacian-based neural operators, whose nonlocal layers are constructed using truncated Laplacian eigenfunction expansions.

Laplacian eigenfunctions have recently been incorporated into operator-learning architectures; see, for example, \cite{hao2025laplacian,chen2023learning}.
These approaches demonstrate the potential benefit of geometry- and boundary-aware spectral representations.
Our focus is complementary: rather than only using Laplacian eigenfunctions as an architectural choice, we study quantitatively why such a representation is well suited to nonlinear parabolic PDE systems.
In particular, the same Laplacian spectrum simultaneously determines the linear diffusion semigroup and the spectral truncation appearing in the neural operator. 
This correspondence allows the approximation complexity of the architecture to be related directly to spectral properties of the underlying PDE.

Our work is also closely related to the quantitative approximation theory of \cite{furuya2024quantitative} for scalar nonlinear parabolic equations.
That work introduced a PDE-specific approximation strategy based on the mild formulation and Picard iteration, but treated the approximation of the underlying integral operator through an abstract finite-rank approximation assumption.
The present work instead considers nonlinear reaction-diffusion systems and makes the spectral approximation explicit by exploiting the Laplacian semigroup.

The distinction from existing Laplacian-based operator-learning results is also worth emphasizing.
For example, the approach of \cite{hao2025laplacian} is based on a DeepONet-type architecture, whereas we analyze a neural operator whose nonlocal layers reproduce truncated Laplacian semigroup and Duhamel propagators.
More importantly, our objective is a quantitative approximation theory that simultaneously controls the network depth, the number of neurons, and the spectral rank required for approximating the full time-dependent solution operator of a nonlinear reaction-diffusion system.

\subsection{Contributions}

Our main contributions are summarized as follows.

\begin{itemize}

\item \textbf{Quantitative approximation theory for nonlinear reaction-diffusion systems.}
We develop a quantitative approximation theory for neural operators applied to nonlinear reaction-diffusion systems.
Rather than treating the solution map as a generic operator, our analysis starts from its mild formulation and constructs the neural operator by approximating the associated Picard iteration.
This provides explicit approximation guarantees for the solution operator mapping initial conditions to full time-dependent trajectories, with network complexity that grows polynomially with respect to the target accuracy.

\item \textbf{Laplacian-based neural operators that exploit PDE structure.}
We construct the nonlocal layers from Laplacian eigenfunctions associated with the boundary conditions of the PDE.
The resulting architecture is directly motivated by the spectral representation of the diffusion semigroup: the basis functions encode the domain geometry and boundary conditions, while the corresponding exponential multipliers encode the diffusion dynamics.
Thus, the architecture incorporates the structure of the underlying PDE at the level of its nonlocal operator.
Numerical experiments support the theoretical motivation and demonstrate the effectiveness of the Laplacian-based representation compared with standard Fourier-based representations.

\end{itemize}

\section{The generalized reaction-diffusion system}\label{sec2}

Let $\Omega$ be a bounded open domain in $\R^n$ whose boundary $\partial\Omega$ is an $(n-1)$-dimensional $C^{2+\alpha}$-manifold for some $\alpha\in(0,1)$ such that $\Omega$ lies locally on one side of $\partial\Omega$. For simplicity, denote $(x,t)=(x_1,\cdots,x_n,t)\in\R^n\times\R$.  Consider $m$-dimensional vector-valued function $u(x,t)=(u_1(x,t),\cdots,u_m(x,t))$ with $m\in\N$ satisfying the following reaction-diffusion system, see \cite[p.108]{rothe2006global},
\begin{equation}\label{main}
u_t-D\Delta u+Ju=F(x,t,u),\quad\forall\;\;x\in\Omega,\; t>0,
\end{equation}
where the reaction function $F:(x,t,u)\in\bar\Omega\times[0,\infty)\times\R^m\to F(x,t,u)\in\R^m$ given by $F=(F_1,\cdots,F_m)(x_1,\cdots,x_n,t,u_1,\cdots,u_m)$. The diffusion coefficients $D$ are defined as follows: $D=\text{diag}(d_1,\cdots,d_m)$ and $J=\text{diag}(\lambda_1,\cdots,\lambda_m)$ with $d_{\ell}>0, \lambda_\ell>0$ for $\ell=1,\cdots,m$. 

We study the initial-boundary value problem for \eqref{main} with homogeneous Dirichlet or Neumann boundary conditions:
\begin{equation}\label{bc}
u_\ell(x,t)=0\quad\text{or}\quad\frac{\partial u_\ell}{\partial\nu}(x,t)=0\quad\text{for}\quad\ell=1,\cdots,m,\;\; x\in\partial\Omega,\; t>0.
\end{equation}
The initial condition is simply given by $u_0(x)=(u_{01},\cdots,u_{0m})(x)$ for $x\in\Omega$.  Inspired by \cite{rothe2006global}, we will consider \emph{mild} and \emph{classical} solutions. For the initial condition 
\begin{equation}\label{3}
u(x,0)=u_0(x)\quad\text{for all}\quad x\in\Omega,
\end{equation}
we assume a general condition
\begin{equation}\label{mo}
u_0\in L^\infty(\Omega,\R^m);
\end{equation}
while, for the reaction term, we assume that 
\begin{equation}\label{mfo}
F(x,t,u):\bar\Omega\times[0,\infty)\to F(x,t,u)\in\R^m\; \text{is measurable in}\; (x,t)\; \text{for all}\; u\in\R^m. 
\end{equation}
Furthermore, for every bounded set $B\subset\bar\Omega\times[0,\infty)\times\R^m$, there exists a constant $L(B)$ such that
\begin{align}
&|F(x,t,u)|\le L(B)\quad\text{for all}\quad (x,t,u)\in B,\label{mf1}\\
&|F(x,t,u)-F(x,t,v)|\le L(B)|u-v|\quad\text{for all}\quad (x,t,u), (x,t,v)\in B.\label{mf2}
\end{align}

To study classical solutions, the initial condition is further required to satisfy
\begin{equation}\label{so}
u_{0\ell}\in C^{2+\alpha}(\bar\Omega)\quad\text{for}\quad \ell=1,\cdots,m,
\end{equation} 
and the following compatibility conditions
\begin{equation}\label{comp}
\begin{aligned}
&u_{0\ell}(x)=0,\quad -d_\ell\Delta u_{0\ell}(x)=F_\ell(x,0,u_0(x))\quad\text{for}\quad\ell=1,\cdots,m, x\in\partial\Omega,\;\;\text{or}\\
&\frac{\partial u_{0\ell}}{\partial\nu}(x)=0\quad\text{for}\quad\ell=1,\cdots,m, x\in\partial\Omega.
\end{aligned}
\end{equation}
For the reaction function $F$, \eqref{mf2} is replaced by
\begin{equation}\label{mf2'}
|F(x,t,u)-F(y,s,v)|\le L(B)[|x-y|^\alpha+|t-s|^{\alpha/2}+|u-v|]\quad\text{for all}\quad (x,t,u), (y,s,v)\in B.
\end{equation}

Let $u(x,t)$ be a solution to \eqref{main} with $F\equiv 0$, \eqref{bc}, and the initial condition $u_0\in L^\infty(\Omega,\R^m)$. We define a formal semigroup $P(t)$ in the space $L^\infty(\Omega,\R^m)$ as
\[
u(x,t)=(P(t)u_0)(x)\quad\text{for all}\quad x\in\Omega,\; t\in[0,\infty).
\]
In view of the definition of $D$, $P(t)$ is explicitly given component-wise by
\[
P_\ell(t)=S_\ell(t)\quad\text{for}\quad\ell=1,\cdots,m,
\]
where $S_\ell(t)$ is the semigroup generated by the operator $-d_\ell\Delta+\lambda_\ell$ and the boundary conditions \eqref{bc}.
Note that $P(t)$ is not a strongly continuous semigroup in $L^\infty(\Omega,\R^m)$.  We now give the definition of mild solutions.
\begin{definition}
Let $T\in(0,\infty]$. An $E_{\infty, 0,T}$-mild solution of the initial-boundary value problem \eqref{main}, \eqref{bc}, and the initial condition $u_0\in L^\infty(\Omega,\R^m)$ on the time interval $[0,T)$ is a measurable function $u:(x,t)\in\Omega\times(0,T)\to u(x,t)\in\R^m$ satisfying
\begin{equation}\label{5}
u(\cdot,t)\in L^\infty(\Omega,\R^m)\quad\text{and}\quad \sup\{\|u(\cdot,s)\|_\infty: s\in(0,t)\}<\infty\quad\text{for all}\quad t\in(0,T),
\end{equation}
\begin{equation}\label{6}
u(\cdot,t)=P(t)u_0+\int_0^tP(t-s)F(\cdot,s,u(\cdot,s))ds\quad\text{for all}\quad t\in(0,T),
\end{equation}
where the integral is an absolutely converging Bochner integral in $L^\infty(\Omega,\R^m)$.
\end{definition}

The existence of mild and classical solutions is stated in the following theorem.
\begin{theorem}\label{thm1}
Assume that $\partial\Omega\in C^{2+\alpha}$, \eqref{mo}, \eqref{mfo}, \eqref{mf1}, \eqref{mf2} hold. Then we have
\begin{enumerate}
\item[(i)] For each $u_0\in L^\infty(\Omega,\R^m)$, there exists $T\in (0,\infty]$ such that the initial-boundary value problem \eqref{main}, \eqref{bc}, \eqref{3} has a unique $E_{\infty,0,T}$-mild solution on the interval $[0,T)$.
\item[(ii)] Let $T=T(u_0)$, then 
\begin{equation}\label{7}
\inf\{T(u_0): u_0\in L^\infty(\Omega,\R^m), \|u_0\|_\infty\le U_0\}>0\quad\text{for all}\quad U_0\in[0,\infty).
\end{equation}
\item[(iii)] The existence time $T\in (0,\infty]$ can be chosen maximal, e.g., (i) does not hold for a larger time. In that case, let $T=T_{\max}$, then
\[
\lim_{t\to T_{\max}}\|u(t)\|_\infty=\infty\quad\text{if}\quad T_{\max}<\infty. 
\]
\end{enumerate}

Next, assume that the initial condition $u_0$ satisfies \eqref{so}, \eqref{comp} and the reaction term $F$ satisfies \eqref{mf1}, \eqref{mf2'}. Then the mild solution $u(x,t)$ is smooth and
\begin{enumerate}
\item[(iv)] 
\begin{equation}\label{10}
u_\ell\in C^{2+\alpha,1+\alpha/2}(\bar\Omega\times(0,T))\quad\text{for}\;\; \ell=1,\cdots,m,\;\text{all}\;\; T\in(0,T_{\max}),
\end{equation}
where $C^{\alpha, \beta}(\bar\Omega\times(0,T))$ for $\alpha, \beta>0$ are parabolic H\"older spaces. 
\end{enumerate}
\end{theorem}
\begin{proof}
Let $U_0\in[0,\infty)$ be given and choose arbitrary $U\in(U_0,\infty)$, $T_0\in(0,\infty)$. Define the bounded set $B=\bar\Omega\times[0,T_0]\times[-U,U]^m$. Let $L(B)$ be the constant appearing in \eqref{mf1}, \eqref{mf2} and choose $T\in(0,T_0]$ such that
\begin{equation}\label{11}
U_0+e^{L(B)T}-1\le U.
\end{equation}
For initial condition $u_0\in L^\infty(\Omega,\R^m)$ with $\|u_0\|_\infty\le U_0$, define the following Picard-Lindel\"of sequence $u^{(i)}\in L^\infty(\bar\Omega\times[0,T],\R^m)$:
\begin{equation}\label{12}
\begin{aligned}
&u^{(1)}(t)=P(t)u_0,\\
&u^{(i+1)}(t)=P(t)u_0+\int_0^t P(t-s)F(\cdot,s,u^{(i)}(\cdot,s))ds,\;\; i\in\mathbb N, \; t\in[0,T].
\end{aligned}
\end{equation}
Next, we define $\eta_i(t):t\in[0,T]\to\eta_i(t)\in[0,\infty)$ by
\[
\eta_i(t)=\|(u^{(i+1)}-u^{(i)})(t)\|_\infty. 
\]
By induction in $i$, we can show the following inequalities: for $t\in[0,T]$
\begin{align}
&\|u^{(i)}(t)\|_\infty\le U,\label{14}\\
&\eta_i(t)\le L(B)\int_0^t\eta_{i-1}(s)ds\quad\text{for}\quad i>1,\label{15}\\
&\eta_i(t)\le(L(B)t)^i/i!,\label{16}\\
&\sum_{1\le j\le i}\eta_j(t)\le e^{L(B)t}-1.\label{17}
\end{align}
It follows from \eqref{14}-\eqref{17} that there exists a function $u\in L^\infty(\bar\Omega\times[0,T],\R^m)$ such that
\begin{equation}\label{remainder}
\sup\{\|(u^{(i)}-u)(t)\|_\infty: t\in[0,T]\}\le\sum_{i\le j<\infty}\eta_j(T)\le\frac{e^{L(B)T}}{\sqrt{2\pi i}}\left(\frac{eL(B)T}{i}\right)^i\to 0\quad\text{for}\quad i\to\infty. 
\end{equation}
\end{proof}

\section{Neural operators and quantitative approximation theorem}

Given $U_0>0$ and define ${\mathcal X}:=\{u_0\in L^\infty(\Omega,\R^m): \|u_0\|_\infty\le U_0\}$. By the well-posedness results established in the previous section, we define a local-in-time solution operator on $\mathcal{X}$ by
\[
\Gamma^+ : \mathcal{X} \to L^{\infty}([0,T];L^{\infty}(\Omega,\R^m)),
\qquad
u_0 \mapsto u,
\]
for some $T>0$, where $u$ denotes the mild solution to
\eqref{main}, \eqref{bc} with initial condition $u_0$. 
The goal of this section is to approximate the solution operator $\Gamma^+$ by a class of neural operators and to establish quantitative error bounds.

\subsection{Neural operator architectures}

Let $\{(\mu_m,\phi_m)\}_{m=1}^{\infty}$ denote the Dirichlet or Neumann eigenpairs of $-\Delta$ on $\Omega$, ordered as $0\le \mu_1<\mu_2\le\mu_3\le\cdots$.

\begin{definition}[Neural operator]
\label{def:neural-operator}
A neural operator
\[
\Gamma : L^{\infty}(\Omega,\R^m)
\longrightarrow L^{\infty}([0,T];L^{\infty}(\Omega,\R^m))
\]
is defined by $\Gamma(u_0)=U^{(L+1)}$, where the output $U^{(L+1)}$ is generated through the following layers.

\medskip

\noindent
1. {\bf (Input layer)} 
$U^{(1)}$ is given by 
\[
U^{(1)}(x,t):= (K^{(0)}u_{0})(x,t)+b^{(0)}(x,t), 
\]
where 
$K^{(0)} : L^{\infty}(\Omega,\R^m)\to L^{\infty}([0,T], L^{\infty}(\Omega,\R^{r_1}))$ is defined by 
\begin{align*}
& 
(K^{(0)} u_0)(x,t) 
:= C^{(0)}\sum_{j \leq N}  E_j(t) \left(\int_{\Omega} u_{0}(y) \phi_j(y) dy \right) \phi_j(x),
\end{align*}
where $C^{(0)} \in \mathbb{R}^{r_1 \times m}$ and 
\[
E_{j}(t) := 
\begin{pmatrix}
e^{-t(d_1 \mu_j + \lambda_1)}&0  & 0 
\\0&\ddots&0\\
0 & 0&e^{-t(d_m \mu_j + \lambda_m)}
\end{pmatrix} \in \mathbb{R}^{m \times m},
\]
where $r_1$ is a suitably chosen dimension.
Here, $b^{(0)} \in L^{\infty}([0,T], L^{\infty}(\Omega,\R^{r_1}))$ is a prescribed positional encoding rather than a trainable bias. 

%
%
%

\medskip\noindent
2. {\bf (Hidden layers)} 
    For $1 \le \ell \le L-1$, $U^{(\ell)}$ are iteratively given by
\begin{align*}
&
U^{(\ell+1)}(x,t)
:= \sigma \left( W^{(\ell)} U^{(\ell)}(x,t) +  (K^{(\ell)} U^{(\ell)})(x,t) 
+ b^{(\ell)}
\right).
\end{align*}
\vspace{0mm}
\\
3. {\bf (Output layer)} 
$U^{(L+1)}$ is given by 
\begin{align*}
&U^{(L+1)}(x,t)
: = W^{(L)} U^{(L)}(x,t) + (K^{(L)}U^{(L)})(x,t)
+b^{(L)}.
\end{align*}
Here, $\sigma:\mathbb{R}\to\mathbb{R}$ denotes a nonlinear activation function applied element-wise.
For each hidden layer $\ell$, $W^{(\ell)}\in\mathbb{R}^{r_{\ell+1}\times r_\ell}$ is a weight matrix and $b^{(\ell)}\in\mathbb{R}^{r_{\ell+1}}$ is a bias vector.
The operator
\[
K^{(\ell)}:
L^{\infty}([0,T];L^{\infty}(\Omega,\R^{r_\ell}))
\to
L^{\infty}([0,T];L^{\infty}(\Omega,\R^{r_{\ell+1}}))
\]
is a nonlocal operator defined by
\begin{align*}
& 
(K^{(\ell)} \tilde{u})(x,t) 
:=  \sum_{j \leq N} C^{(\ell,1)}
\left(\int_0^{t} \int_{\Omega} E_{j}(t-s)C^{(\ell,2)} \tilde{u}(y,s) \phi_j(y) dyds\right) \phi_j(x), 
\end{align*}
where $C^{(\ell,1)}\in\mathbb{R}^{r_{\ell+1}\times m}$ and $C^{(\ell,2)}\in\mathbb{R}^{m\times r_\ell}$ are trainable matrices.
By construction, $r_{L+1}=m$, so that the output dimension matches the $m$-component state variable $u$.

The collection of learnable parameters is given by
\[
\{W^{(\ell)},\, b^{(\ell)},\, C^{(0)},\, C^{(\ell,1)},\, C^{(\ell,2)}\}_{\ell=1}^L.
\]
We denote by $\mathcal{NO}^{L,H,N}_{\sigma}$ the class of neural operators defined above, characterized by the depth $L$, the total number of neurons $H=\sum_{\ell=1}^L r_\ell$, the spectral rank $N$, and the activation function $\sigma$.
\end{definition}

This neural operator exploits the Laplacian eigenfunction basis through the nonlocal operators $K^{(\ell)}$. 
In contrast to standard FNOs, which are tailored to periodic domains, this architecture explicitly incorporates geometric information about the domain as well as Dirchlet or Neumann boundary conditions. 

Since we are primarily interested in approximating the solution operator $\Gamma^+$ for initial data $u_0\in\mathcal X$, we introduce a globally bounded and globally Lipschitz extension of the reaction term in the state variable.

Let $U>0$ be chosen so that every mild solution corresponding to $u_0\in\mathcal X$ satisfies
\[
\|u(t)\|_{L^\infty(\Omega,\mathbb R^m)}
\le U,
\qquad t\in[0,T_0],
\]
where $T_0>0$ will be chosen below.
Define the componentwise clipping map
\[
\mathcal C_U:\mathbb R^m\to[-U,U]^m
\]
by
\[
(\mathcal C_U(z))_j
:=
\max\{-U,\min\{z_j,U\}\},
\qquad j=1,\ldots,m.
\]
We then define
\begin{equation}
\label{eq:F-extension}
\widetilde F(x,t,z)
:=
F(x,t,\mathcal C_U(z)).
\end{equation}
Since $\mathcal C_U$ is $1$-Lipschitz, any Lipschitz bound for $F$ on
\[
\bar\Omega\times[0,T_0]\times[-U,U]^m
\]
extends to $\widetilde F$ on $\bar\Omega\times[0,T_0]\times\mathbb R^m$ with the same constant.
Moreover, $\widetilde F$ is uniformly bounded in the state variable.
Finally,
\[
\widetilde F(x,t,z)=F(x,t,z)
\qquad\text{for }z\in[-U,U]^m,
\]
so replacing $F$ by $\widetilde F$ does not change any of the mild solutions under consideration.


\bigskip

We now state the main result.

\begin{theorem}
\label{thm:main}
Let $n\ge1$, and let $\Omega\subset\mathbb R^n$ be a bounded domain with $C^{2+\alpha}$ boundary.
Let $\tilde T>0$.
Assume that the clipped extension $\widetilde F$ defined in \eqref{eq:F-extension} satisfies
\begin{equation}
\label{eq:F-global-bound}
|\widetilde F(x,t,u)|
\le M_F
\end{equation}
and
\begin{equation}
\label{mf3}
|\widetilde F(x,t,u)-\widetilde F(y,s,v)|
\le
L_F
\bigl(
|x-y|+|t-s|+|u-v|
\bigr)
\end{equation}
for all
\[
(x,t),(y,s)\in\bar\Omega\times[0,\tilde T],
\qquad
u,v\in\mathbb R^m,
\]
where $M_F,L_F>0$ are independent of $u$ and $v$.
Then there exists $T_0\in(0,\tilde T]$ such that, for any $\epsilon\in(0,1)$, there exists a neural operator
\[
\Gamma\in\mathcal{NO}^{L,H,N}_{\mathrm{ReLU}}
\]
satisfying
\[
\sup_{u_0\in\mathcal X}
\|\Gamma(u_0)-\Gamma^+(u_0)\|_
{L^2(0,T_0;L^2(\Omega,\mathbb R^m))}
\le\epsilon.
\]
Moreover,
\[
L(\Gamma)
\le
C(\log\epsilon^{-1})^2,
\qquad
H(\Gamma)
\le
C\epsilon^{-(n+m+1)}
(\log\epsilon^{-1})^2,
\]
and
\[
N(\Gamma)
\le
C\epsilon^{-n},
\]
where $C>0$ is independent of $\epsilon$.
\end{theorem}

Theorem~\ref{thm:main} shows that the depth grows at most squared logarithmically, and the total number of neurons and the spectral rank grows at most polynomially as $\epsilon\to0$. 
We emphasize that the Laplacian eigenfunction-based neural operator avoids the curse of parametric complexity, namely the exponential growth of parameters observed for general operator learning \cite{lanthaler2023curse}. 
This theoretical finding is further supported by numerical experiments, where we compare the Laplacian eigenfunction-based neural operator with standard FNOs.

\section{Proof of Theorem~\ref{thm:main}}
\label{app:main-proof}

We first introduce some notation. 
Let
\[
Y:=L^2(0,T_0;L^2(\Omega,\mathbb R^m)).
\]
For an operator $A:X\to X$ and $k\in\mathbb N_0$, we denote by $A^{[k]}$ the $k$-fold composition of $A$, with the convention $A^{[0]}=\operatorname{Id}$. For $j=1,\ldots,m$, set
\[
A_j:=-d_j\Delta+\lambda_j
\]
with the boundary condition \eqref{bc}.
Each $A_j$ is a positive self-adjoint operator on $L^2(\Omega)$.
If $\{(\mu_\ell,\phi_\ell)\}_{\ell=1}^\infty$ denotes the Dirichlet or Neumann eigenpairs of $-\Delta$, then
\[
A_j\phi_\ell=(d_j\mu_\ell+\lambda_j)\phi_\ell.
\]
In particular,
\[
\|e^{-tA_j}\|_{L^2(\Omega)\to L^2(\Omega)}
\le e^{-\lambda_jt}\le1.
\]
Let $\Pi_N$ denote the orthogonal projection onto
$\operatorname{span}\{\phi_1,\ldots,\phi_N\}$ in $L^2(\Omega)$, acting componentwise on $L^2(\Omega,\mathbb R^m)$, and define
\[
P_N(t):=\Pi_NP(t)=P(t)\Pi_N.
\]

Let $\epsilon\in(0,1)$ and $u_0\in\mathcal X$.
We choose $T_0\in(0,\tilde T]$ sufficiently small, independently of $u_0$, so that the mild solution corresponding to every $u_0\in\mathcal X$ remains in the region where $\tilde F=F$ and, in addition, the contraction condition specified below is satisfied.

\medskip\medskip

\noindent{\bf Step 1. Spectral truncation estimates in $L^2$.}

Define
\[
\Theta_{u_0}:Y\to Y
\]
by
\begin{equation}
\label{eq:Theta}
\Theta_{u_0}[u](t)
:=
P(t)u_0+
\int_0^tP(t-s)\tilde F(\cdot,s,u(s))\,ds,
\end{equation}
and its spectral truncation
\[
\Theta_{u_0,N}:Y\to Y
\]
by
\begin{equation}
\label{eq:Theta-N}
\Theta_{u_0,N}[u](t)
:=
P_N(t)u_0+
\int_0^tP_N(t-s)\tilde F(\cdot,s,u(s))\,ds.
\end{equation}

We first estimate the truncation error of the homogeneous term.
Writing
\[
u_{0,j}
=
\sum_{\ell=1}^\infty
u_{0,j,\ell}\phi_\ell,
\qquad
u_{0,j,\ell}
=
\langle u_{0,j},\phi_\ell\rangle_{L^2(\Omega)},
\]
Parseval's identity gives
\begin{align}
\|(P-P_N)u_0\|_Y^2
\nonumber&=
\sum_{j=1}^m
\sum_{\ell>N}
|u_{0,j,\ell}|^2
\int_0^{T_0}
e^{-2(d_j\mu_\ell+\lambda_j)t}\,dt
\nonumber\\
&\le
\sum_{j=1}^m
\frac{1}{2(d_j\mu_{N+1}+\lambda_j)}
\sum_{\ell>N}|u_{0,j,\ell}|^2
\nonumber\\
&\le
C(1+\mu_{N+1})^{-1}
\|u_0\|_{L^2(\Omega,\mathbb R^m)}^2.
\label{eq:L2-initial-tail}
\end{align}
Consequently,
\begin{equation}
\label{eq:L2-initial-tail-2}
\|(P-P_N)u_0\|_Y
\le
C(1+\mu_{N+1})^{-1/2}
\|u_0\|_{L^2(\Omega,\mathbb R^m)}.
\end{equation}

We next estimate the truncation error of the Duhamel term.
For $g\in Y$, define
\[
(\mathcal I g)(t)
:=
\int_0^tP(t-s)g(s)\,ds,
\qquad
(\mathcal I_N g)(t)
:=
\int_0^tP_N(t-s)g(s)\,ds.
\]
For each component, write
\[
g_j(t)
=
\sum_{\ell=1}^\infty g_{j,\ell}(t)\phi_\ell.
\]
Then
\[
((\mathcal I-\mathcal I_N)g)_j(t)
=
\sum_{\ell>N}
\left(
\int_0^t
e^{-(d_j\mu_\ell+\lambda_j)(t-s)}
g_{j,\ell}(s)\,ds
\right)\phi_\ell.
\]
By Young's inequality for time convolution,
\begin{align*}
\left\|
\int_0^t
e^{-(d_j\mu_\ell+\lambda_j)(t-s)}
g_{j,\ell}(s)\,ds
\right\|_{L^2(0,T_0)}&\le
\|e^{-(d_j\mu_\ell+\lambda_j)t}\|_{L^1(0,T_0)}
\|g_{j,\ell}\|_{L^2(0,T_0)}
\\
&\le
\frac{1}{d_j\mu_\ell+\lambda_j}
\|g_{j,\ell}\|_{L^2(0,T_0)}.
\end{align*}
Therefore, Parseval's identity yields
\begin{equation}
\label{eq:L2-Duhamel-tail}
\|(\mathcal I-\mathcal I_N)g\|_Y
\le
C(1+\mu_{N+1})^{-1}\|g\|_Y.
\end{equation}
By Weyl's law,
\[
\mu_{N+1}\simeq N^{2/n}.
\]
Hence
\begin{equation}
\label{eq:L2-spectral-rates}
\|(P-P_N)u_0\|_Y
\lesssim
N^{-1/n}\|u_0\|_{L^2(\Omega,\mathbb R^m)},
\end{equation}
and
\begin{equation}
\label{eq:L2-Duhamel-rate}
\|(\mathcal I-\mathcal I_N)g\|_Y
\lesssim
N^{-2/n}\|g\|_Y.
\end{equation}

Since $\mathcal X$ is bounded in $L^\infty(\Omega,\mathbb R^m)$, it is uniformly bounded in $L^2(\Omega,\mathbb R^m)$.
Moreover, by \eqref{eq:F-global-bound},
\[
\|\widetilde F(\cdot,\cdot,u)\|_Y
\le
M_F|\Omega|^{1/2}T_0^{1/2}
\]
for every $u\in Y$.
Therefore, \eqref{eq:L2-spectral-rates} and
\eqref{eq:L2-Duhamel-rate} imply
\begin{equation}
\label{eq:Theta-truncation}
\sup_{u_0\in\mathcal X}
\|\Theta_{u_0}[u]-\Theta_{u_0,N}[u]\|_Y
\le
C\bigl(N^{-1/n}+N^{-2/n}\bigr)
\le
C N^{-1/n},
\end{equation}
uniformly for $u\in Y$.

We therefore choose
\begin{equation}
\label{N}
N\asymp\epsilon^{-n}.
\end{equation}
Then
\begin{equation}
\label{eq:Theta-truncation-epsilon}
\sup_{u_0\in\mathcal X}
\|\Theta_{u_0}[u]-\Theta_{u_0,N}[u]\|_Y
\le C\epsilon
\end{equation}
uniformly for $u\in Y$.

\medskip\medskip

\noindent{\bf Step 2. Approximation of the reaction term.}

Set
\[
\Lambda
:=
\bar\Omega\times[0,T_0]\times[-U,U]^m.
\]
By assumption, the restriction of $F$ to $\Lambda$ belongs to $W^{1,\infty}(\Lambda,\mathbb R^m)$.
Applying \cite[Theorem~1]{yarotsky2017error}, there exists a ReLU neural network
\[
G_{\mathrm{net}}
:
\mathbb R^{n+1+m}\to\mathbb R^m
\]
such that
\begin{equation}
\label{eq:approx-NN-nonlinearity-local}
\sup_{(x,t,z)\in\Lambda}
|F(x,t,z)-G_{\mathrm{net}}(x,t,z)|
\le\epsilon,
\end{equation}
with
\begin{equation}
\label{eq:quanti-esti-nonlinearity}
L(G_{\mathrm{net}})
\lesssim
\log(\epsilon^{-1}),
\qquad
H(G_{\mathrm{net}})
\lesssim
\epsilon^{-(n+m+1)}
\log(\epsilon^{-1}).
\end{equation}

Recall the componentwise clipping map $\mathcal C_U$ defined above.
Since each scalar clipping function can be represented exactly by a
ReLU network with a constant number of layers and neurons, the map
$\mathcal C_U$ admits an exact ReLU realization whose complexity
depends only on $m$.

Define
\begin{equation}
\label{eq:Fnet-global}
\widetilde F_{\mathrm{net}}(x,t,z)
:=
G_{\mathrm{net}}(x,t,\mathcal C_U(z)).
\end{equation}
Then $\widetilde F_{\mathrm{net}}$ is a ReLU neural network and
\begin{equation}
\label{eq:approx-NN-nonlinearity}
\sup_{\substack{(x,t)\in\bar\Omega\times[0,T_0]\\
z\in\mathbb R^m}}
|\widetilde F(x,t,z)
-\widetilde F_{\mathrm{net}}(x,t,z)|
\le\epsilon.
\end{equation}
Moreover,
\[
L(\widetilde F_{\mathrm{net}})
\lesssim
\log(\epsilon^{-1}),
\qquad
H(\widetilde F_{\mathrm{net}})
\lesssim
\epsilon^{-(n+m+1)}
\log(\epsilon^{-1}).
\]

Define
\[
\Theta_{u_0,N,\mathrm{net}}:Y\to Y
\]
by
\begin{equation}
\label{eq:Theta-N-net}
\Theta_{u_0,N,\mathrm{net}}[u](t)
:=
P_N(t)u_0
+
\int_0^t
P_N(t-s)
\widetilde F_{\mathrm{net}}(\cdot,s,u(s))\,ds.
\end{equation}

Since $P(t)$ is contractive on $L^2(\Omega,\mathbb R^m)$ and $\widetilde F$ is globally Lipschitz in its last variable with Lipschitz constant $L_F$, we have
\begin{align}
&
\|\Theta_{u_0}[u]-\Theta_{u_0}[v]\|_Y
\nonumber\\
&\le
L_F
\left\|
\int_0^t
\|u(s)-v(s)\|_{L^2(\Omega,\mathbb R^m)}
\,ds
\right\|_{L^2(0,T_0)}
\nonumber\\
&\le
L_FT_0\|u-v\|_Y.
\label{eq:L2-contraction}
\end{align}
We choose $T_0>0$ sufficiently small so that
\begin{equation}
\label{eq:delta}
\delta:=L_FT_0<1.
\end{equation}

We next establish the one-step approximation estimate
\begin{equation}
\label{eq:one-step-L2}
\|\Theta_{u_0}[u]
-
\Theta_{u_0,N,\mathrm{net}}[v]\|_Y
\le
C\epsilon+\delta\|u-v\|_Y
\end{equation}
for arbitrary $u,v\in Y$.
Indeed,
\begin{align*}
&
\|\Theta_{u_0}[u]
-
\Theta_{u_0,N,\mathrm{net}}[v]\|_Y
\\
&\le
\|\Theta_{u_0}[u]-\Theta_{u_0}[v]\|_Y
+
\|\Theta_{u_0}[v]-\Theta_{u_0,N}[v]\|_Y
\\
&\quad+
\|\Theta_{u_0,N}[v]
-\Theta_{u_0,N,\mathrm{net}}[v]\|_Y.
\end{align*}
The first term is bounded by $\delta\|u-v\|_Y$ by \eqref{eq:L2-contraction}, and the second is bounded by $C\epsilon$ by \eqref{eq:Theta-truncation-epsilon}.

For the last term, the $L^2$-contractivity of $P_N(t)$ and the global approximation estimate \eqref{eq:approx-NN-nonlinearity} yield
\begin{align*}
&
\|\Theta_{u_0,N}[v]
-\Theta_{u_0,N,\mathrm{net}}[v]\|_Y
\\
&\le
\left\|
\int_0^t
\|\widetilde F(\cdot,s,v(s))
-\widetilde F_{\mathrm{net}}(\cdot,s,v(s))\|_
{L^2(\Omega,\mathbb R^m)}
\,ds
\right\|_{L^2(0,T_0)}
\\
&\le
T_0|\Omega|^{1/2}\epsilon.
\end{align*}
Thus \eqref{eq:one-step-L2} follows.

\medskip\medskip

\noindent{\bf Step 3. Picard iteration and approximation error.}

Set
\[
u^{(0)}(t):=P(t)u_0,
\qquad
u_N^{(0)}(t):=P_N(t)u_0,
\]
and define
\[
u^{(k)}
:=
\Theta_{u_0}^{[k]}[u^{(0)}],
\qquad
u_{N,\mathrm{net}}^{(k)}
:=
\Theta_{u_0,N,\mathrm{net}}^{[k]}[u_N^{(0)}].
\]
For a positive integer $k$ to be specified below, define
\begin{equation}
\label{def:Gamma}
\Gamma(u_0)
:=
u_{N,\mathrm{net}}^{(k)}.
\end{equation}

Since $\Theta_{u_0}$ is a contraction on $Y$ with contraction constant $\delta<1$, its fixed point is the mild solution $\Gamma^+(u_0)$. 
Therefore,
\begin{equation}
\label{eq:Picard-exact}
\|\Gamma^+(u_0)-u^{(k)}\|_Y
\le
C\delta^k
\end{equation}
uniformly for $u_0\in\mathcal X$.

On the other hand, by \eqref{eq:one-step-L2},
\begin{align*}
\|u^{(k)}-u_{N,\mathrm{net}}^{(k)}\|_Y
&\le
C\epsilon
+
\delta
\|u^{(k-1)}-u_{N,\mathrm{net}}^{(k-1)}\|_Y
\\
&\le
C\epsilon\sum_{j=0}^{k-1}\delta^j
+
\delta^k
\|u^{(0)}-u_N^{(0)}\|_Y
\\
&\le
\frac{C\epsilon}{1-\delta}
+
C\delta^kN^{-1/n},
\end{align*}
where we used \eqref{eq:L2-initial-tail-2} and the uniform $L^2$-boundedness of $\mathcal X$.

Combining this estimate with \eqref{eq:Picard-exact}, we obtain
\begin{equation}
\label{eq:total-L2-error}
\|\Gamma^+(u_0)-\Gamma(u_0)\|_Y
\le
C\delta^k+C\epsilon.
\end{equation}

Choose
\[
k
=
\left\lceil
\frac{\log(C\epsilon^{-1})}{|\log\delta|}
\right\rceil.
\]
Then
\[
k\lesssim\log(\epsilon^{-1}),
\qquad
\delta^k\lesssim\epsilon.
\]
After replacing $\epsilon$ in the preceding construction by a sufficiently small constant multiple of the prescribed accuracy, we obtain
\[
\sup_{u_0\in\mathcal X}
\|\Gamma^+(u_0)-\Gamma(u_0)\|_
{L^2(0,T_0;L^2(\Omega,\mathbb R^m))}
\le\epsilon.
\]

\medskip\medskip

\noindent{\bf Final step.}
Finally, it remains to represent the approximate operator $\Gamma$ defined in \eqref{def:Gamma} as a neural operator in the sense of Definition~\ref{def:neural-operator}, and to derive quantitative estimates of its complexity.

\begin{lemma}
\label{lem:representation-NO}
Let $\Gamma$ be the map defined by \eqref{def:Gamma}, where
\[
k\asymp \log(\epsilon^{-1}).
\]
Then
\[
\Gamma\in\mathcal{NO}^{L,H,N}_{\mathrm{ReLU}},
\]
where $L=L(\Gamma)$, $H=H(\Gamma)$, and $N=N(\Gamma)$ satisfy
\[
L(\Gamma)
\le
C(\log(\epsilon^{-1}))^2,
\]
\[
H(\Gamma)
\le
C\epsilon^{-(n+m+1)}
(\log(\epsilon^{-1}))^2,
\]
and
\[
N(\Gamma)
\le
C\epsilon^{-n}.
\]
\end{lemma}

\begin{proof}
Recall that
\[
u_N^{(0)}(x,t)
:=
P_N(t)u_0(x)
=
\sum_{\ell=1}^N
E_\ell(t)
\left(
\int_\Omega u_0(y)\phi_\ell(y)\,dy
\right)
\phi_\ell(x),
\]
where
\[
E_\ell(t)
=
\operatorname{diag}
\left(
e^{-t(d_1\mu_\ell+\lambda_1)},
\ldots,
e^{-t(d_m\mu_\ell+\lambda_m)}
\right).
\]
Moreover, the approximate Picard iterates are recursively defined by
\begin{equation}
\label{eq:approx-Picard-final}
u_N^{(q+1)}(x,t)
=
u_N^{(0)}(x,t)
+
\int_0^t
P_N(t-s)
\tilde F_{\mathrm{net}}
\bigl(\cdot,s,u_N^{(q)}(\cdot,s)\bigr)(x)
\,ds,
\qquad q\ge0.
\end{equation}
Equivalently,
\begin{align}
u_N^{(q+1)}(x,t)
&=
u_N^{(0)}(x,t)
+
\sum_{\ell=1}^N
\int_0^t
E_\ell(t-s)
\left(
\int_\Omega
\tilde F_{\mathrm{net}}
\bigl(y,s,u_N^{(q)}(y,s)\bigr)
\phi_\ell(y)\,dy
\right)
\phi_\ell(x)\,ds.
\label{eq:approx-Picard-spectral}
\end{align}
By definition,
\[
\Gamma(u_0)=u_N^{(k)}.
\]

We show that the recursion \eqref{eq:approx-Picard-spectral} can be realized by the neural operator architecture of Definition~\ref{def:neural-operator}.

\medskip

\noindent
\emph{Input layer.}
We use $2m+n+1$ channels and define
\[
K^{(0)}
:
L^\infty(\Omega,\mathbb R^m)
\longrightarrow
L^\infty
\bigl(
[0,T_0];
L^\infty(\Omega,\mathbb R^{2m+n+1})
\bigr)
\]
by
\begin{equation}
\label{eq:K0-final}
(K^{(0)}u_0)(x,t)
:=
\begin{pmatrix}
u_N^{(0)}(x,t)
\\[0.2em]
u_N^{(0)}(x,t)
\\[0.2em]
0
\end{pmatrix}
=
\begin{pmatrix}
\displaystyle
\sum_{\ell=1}^N
E_\ell(t)
\langle u_0,\phi_\ell\rangle_{L^2(\Omega)}
\phi_\ell(x)
\\[0.8em]
\displaystyle
\sum_{\ell=1}^N
E_\ell(t)
\langle u_0,\phi_\ell\rangle_{L^2(\Omega)}
\phi_\ell(x)
\\[0.8em]
0
\end{pmatrix}.
\end{equation}
Here the last block has dimension $n+1$.

We also define the positional bias
\[
b^{(0)}(x,t)
:=
\begin{pmatrix}
0\\
0\\
x\\
t
\end{pmatrix}
\in\mathbb R^{2m+n+1}.
\]
Hence, the output of the input layer is
\begin{equation}
\label{eq:input-augmented}
U^{(1)}(x,t)
=
(K^{(0)}u_0)(x,t)+b^{(0)}(x,t)
=
\begin{pmatrix}
u_N^{(0)}(x,t)
\\
u_N^{(0)}(x,t)
\\
x\\
t
\end{pmatrix}.
\end{equation}
The first block will contain the current Picard iterate, the second block stores the homogeneous solution $u_N^{(0)}$, which remains unchanged throughout the iteration, and the last $n+1$ channels store the spatial and temporal coordinates.

\medskip

\noindent
\emph{Local reaction block.}
For
\[
U=(U_1,U_2,U_3)
\in
\mathbb R^m\times\mathbb R^m\times\mathbb R^{n+1},
\]
write
\[
U_3=(x,t).
\]
Define
\[
\tilde F'_{\mathrm{net}}
:
\mathbb R^{2m+n+1}
\longrightarrow
\mathbb R^{2m+n+1}
\]
by
\begin{equation}
\label{eq:Fprime-net}
\tilde F'_{\mathrm{net}}(U)
:=
\begin{pmatrix}
\tilde F_{\mathrm{net}}(x,t,U_1)
\\
U_2
\\
U_3
\end{pmatrix}.
\end{equation}
Since $\tilde F_{\mathrm{net}}$ is a ReLU neural network and the identity map can also be represented by a ReLU neural network, $\tilde F'_{\mathrm{net}}$ is a ReLU neural network.
Moreover,
\begin{equation}
\label{eq:Fprime-complexity}
L(\tilde F'_{\mathrm{net}})
\lesssim
L(\tilde F_{\mathrm{net}}),
\qquad
H(\tilde F'_{\mathrm{net}})
\lesssim
H(\tilde F_{\mathrm{net}})+m+n+1.
\end{equation}

\medskip

\noindent
\emph{Truncated Duhamel layer.}
Define
\[
W
:=
\begin{pmatrix}
0&I_m&0\\
0&I_m&0\\
0&0&I_{n+1}
\end{pmatrix}
\in
\mathbb R^{(2m+n+1)\times(2m+n+1)}.
\]
We also define the nonlocal operator
\[
K:
L^\infty
\bigl(
[0,T_0];
L^\infty(\Omega,\mathbb R^{2m+n+1})
\bigr)
\longrightarrow
L^\infty
\bigl(
[0,T_0];
L^\infty(\Omega,\mathbb R^{2m+n+1})
\bigr)
\]
by
\begin{equation}
\label{eq:K-final}
(KU)(x,t)
:=
\begin{pmatrix}
\displaystyle
\sum_{\ell=1}^N
\int_0^t
E_\ell(t-s)
\left(
\int_\Omega
U_1(y,s)\phi_\ell(y)\,dy
\right)
\phi_\ell(x)\,ds
\\[1.2em]
0
\\
0
\end{pmatrix},
\end{equation}
where
\[
U=(U_1,U_2,U_3),
\qquad
U_1,U_2\in\mathbb R^m,
\quad
U_3\in\mathbb R^{n+1}.
\]
The operator $K$ is precisely a nonlocal layer of
Definition~\ref{def:neural-operator}, with spectral rank $N$.
Indeed, it is obtained by choosing the matrices $C^{(\ell,1)}$ and $C^{(\ell,2)}$ so that the nonlocal operator acts only on the first $m$ channels and places its output in the first $m$ channels.

We now verify that one application of
\[
(W+K)\circ\tilde F'_{\mathrm{net}}
\]
corresponds exactly to one approximate Picard iteration. Suppose that, for some $q\ge0$,
\[
U(x,t)
=
\begin{pmatrix}
u_N^{(q)}(x,t)
\\
u_N^{(0)}(x,t)
\\
x\\
t
\end{pmatrix}.
\]
Then, by \eqref{eq:Fprime-net},
\[
\tilde F'_{\mathrm{net}}(U)(x,t)
=
\begin{pmatrix}
\tilde F_{\mathrm{net}}
\bigl(x,t,u_N^{(q)}(x,t)\bigr)
\\
u_N^{(0)}(x,t)
\\
x\\
t
\end{pmatrix}.
\]
Applying the pointwise linear map $W$ gives
\[
W\tilde F'_{\mathrm{net}}(U)(x,t)
=
\begin{pmatrix}
u_N^{(0)}(x,t)
\\
u_N^{(0)}(x,t)
\\
x\\
t
\end{pmatrix},
\]
whereas applying $K$ gives
\[
K\tilde F'_{\mathrm{net}}(U)(x,t)
=
\begin{pmatrix}
\displaystyle
\sum_{\ell=1}^N
\int_0^t
E_\ell(t-s)
\left(
\int_\Omega
\tilde F_{\mathrm{net}}
\bigl(y,s,u_N^{(q)}(y,s)\bigr)
\phi_\ell(y)\,dy
\right)
\phi_\ell(x)\,ds
\\[1.2em]
0\\
0
\end{pmatrix}.
\]
Consequently, by \eqref{eq:approx-Picard-spectral},
\begin{align}
\bigl[(W+K)\circ\tilde F'_{\mathrm{net}}\bigr]
\begin{pmatrix}
u_N^{(q)}
\\
u_N^{(0)}
\\
\operatorname{id}_{\Omega\times[0,T_0]}
\end{pmatrix}
(x,t)=
\begin{pmatrix}
u_N^{(q+1)}(x,t)
\\
u_N^{(0)}(x,t)
\\
x\\
t
\end{pmatrix}.
\label{eq:one-Picard-one-block}
\end{align}
Thus one application of the block
\[
(W+K)\circ\tilde F'_{\mathrm{net}}
\]
corresponds exactly to one Picard iteration while preserving both
the homogeneous solution $u_N^{(0)}$ and the positional coordinates
$(x,t)$.

Starting from \eqref{eq:input-augmented} and applying
\eqref{eq:one-Picard-one-block} repeatedly, we obtain by induction
that
\begin{equation}
\label{eq:k-block}
\Bigl[
(W+K)\circ\tilde F'_{\mathrm{net}}
\Bigr]^{[q]}
\circ
(K^{(0)}+b^{(0)})(u_0)
=
\begin{pmatrix}
u_N^{(q)}
\\
u_N^{(0)}
\\
x\\
t
\end{pmatrix}
\end{equation}
for every $q=0,\ldots,k$. It remains to verify that the preceding block construction can be realized within Definition~\ref{def:neural-operator}, in which every hidden affine local/nonlocal layer is followed by the ReLU activation.

For this purpose, we use the standard positive-negative channel representation. 
For every scalar or vector quantity $z$, define
\[
z^+:=\operatorname{ReLU}(z),
\qquad
z^-:=\operatorname{ReLU}(-z),
\]
componentwise, so that
\begin{equation}
\label{eq:signed-representation}
z=z^+-z^-.
\end{equation}
Thus any signed feature can be propagated through ReLU layers by storing the pair $(z^+,z^-)$.

More generally, whenever a hidden affine local/nonlocal layer in Definition~\ref{def:neural-operator} is required to produce a signed feature $z$, we instead let it produce
\[
\begin{pmatrix}
z\\
-z
\end{pmatrix}.
\]
After applying ReLU, the resulting channels store $(z^+,z^-)$, from which $z=z^+-z^-$ is exactly recovered in the next affine layer.

We apply this signed-channel construction to every intermediate layer in the realization of $\widetilde F'_{\mathrm{net}}$ and to the affine nonlocal map $W+K$.
In particular, after evaluating $\widetilde F'_{\mathrm{net}}$, the affine nonlocal layer is chosen to output both
\[
(W+K)\widetilde F'_{\mathrm{net}}(U)
\quad\text{and}\quad
-(W+K)\widetilde F'_{\mathrm{net}}(U).
\]
The following ReLU activation therefore stores the positive and negative parts of
\[
\begin{pmatrix}
u_N^{(q+1)}\\
u_N^{(0)}\\
x\\
t
\end{pmatrix},
\]
and the next affine layer reconstructs these signed quantities before evaluating the next reaction block.

Consequently, each approximate Picard step admits an exact realization by a finite sequence of layers of
Definition~\ref{def:neural-operator}.
The signed-channel representation increases the width by at most a constant factor and adds only a constant number of layers to each Picard block.
These constants depend on $m$ and $n$, but are independent of $\epsilon$, $N$, and $k$.

Starting from the input representation \eqref{eq:input-augmented} and repeating this construction $k$ times therefore gives an exact realization of
\[
\begin{pmatrix}
u_N^{(k)}\\
u_N^{(0)}\\
x\\
t
\end{pmatrix}
\]
within the neural-operator architecture of Definition~\ref{def:neural-operator}.

Finally, define the output projection
\[
W'
:=
\begin{pmatrix}
I_m&0&0
\end{pmatrix}
\in
\mathbb R^{m\times(2m+n+1)}.
\]
Then \eqref{eq:k-block} gives
\begin{equation}
\label{eq:Gamma-NO-representation}
\Gamma(u_0)
=
W'
\circ
\Bigl[
(W+K)\circ\tilde F'_{\mathrm{net}}
\Bigr]^{[k]}
\circ
(K^{(0)}+b^{(0)})(u_0).
\end{equation}
Hence $\Gamma$ belongs to the neural-operator class
\[
\Gamma\in
\mathcal{NO}^{L,H,N}_{\mathrm{ReLU}}.
\]

It remains to estimate its complexity.
By \eqref{eq:quanti-esti-nonlinearity},
\[
L(\tilde F_{\mathrm{net}})
\lesssim
\log(\epsilon^{-1}),
\]
and
\[
H(\tilde F_{\mathrm{net}})
\lesssim
\epsilon^{-(n+m+1)}
\log(\epsilon^{-1}).
\]
Since
\[
k\lesssim\log(\epsilon^{-1}),
\]
and each approximate Picard iteration requires one copy of $\tilde F'_{\mathrm{net}}$ together with only a constant number of additional local and nonlocal layers, we obtain
\begin{align*}
L(\Gamma)\lesssim
k\bigl(
L(\tilde F_{\mathrm{net}})+1
\bigr)\lesssim
(\log(\epsilon^{-1}))^2.
\end{align*}
Similarly,
\begin{align*}
H(\Gamma)\lesssim
k
\bigl(
H(\tilde F_{\mathrm{net}})
+2m+n+1
\bigr)
\lesssim
\epsilon^{-(n+m+1)}
(\log(\epsilon^{-1}))^2.
\end{align*}
Finally, by the choice of the spectral truncation rank in
\eqref{N},
\[
N(\Gamma)=N\lesssim\epsilon^{-n}.
\]
This proves the lemma and completes the proof of
Theorem~\ref{thm:main}.
\end{proof}

\newcommand{\Ryo}[1]{{\color{red}[\textit{#1} Ryo]}}



\section{Numerical example}\label{sec5}

Motivated by the Laplacian spectral representation underlying the architecture analyzed above, we numerically investigate a Laplacian-eigenfunction-based spectral neural operator for a generalized Gierer-Meinhardt system and a Gray-Scott system.
The numerical architecture is not intended to reproduce the time-integrated Picard architecture used in the proof of the main approximation theorem exactly.
Instead, it evaluates the central architectural principle suggested by the theory: replacing a generic Fourier representation by the spectral basis of the Laplacian associated with the geometry and boundary conditions of the underlying PDE.We conduct three experiments: a Gierer-Meinhardt system on a square (Experiment~I), a Gierer-Meinhardt system on a disk (Experiment~II), and a Gray-Scott system on a square (Experiment~III).

Following \cite{kovachki2023neural}, we learn a one-step map from the current state to the next sampled state. 
For computational efficiency, we use this autoregressive formulation in place of the full-trajectory operator considered in the main approximation theorem.
Let
\[
0=t_0<t_1<\cdots<t_K=T
\]
denote the sampling times, and write $z_k=(u_{t_k},v_{t_k})$ for the two-component state, using $(U,V)$ in Experiment~III. For each fixed set of PDE parameters, we train a separate model $\Gamma$ to approximate
\[
\Gamma(z_k)\approx z_{k+1},\qquad k=0,1,\dots,K-1.
\]
The training loss is the relative $L^2$ error:
\begin{equation}
\label{eq:one-step-loss}
\mathcal{L}
=\sum_{k=0}^{K-1}
\frac{\|\Gamma(z_k)-z_{k+1}\|_{L^2(\Omega)^2}}
{\|z_{k+1}\|_{L^2(\Omega)^2}}.
\end{equation}
\[
\hat z_{k+1}=\Gamma(\hat z_k),\qquad k=0,1,\dots,K-1.
\]

We compare the Fourier Neural Operator (FNO), which uses Fourier modes for its spectral convolutions, with the \emph{Laplacian Eigenfunction Neural Operator} (LENO), which uses a spectral basis associated with the Laplacian and the prescribed boundary conditions. 
On the square and the disk domains, we use the analytical Laplacian eigenfunctions for homogeneous
Neumann boundary conditions. In the FNO implementation on these domains, the input fields are padded before applying the Fourier transform.

Let $p=1,\dots,100$ index trajectories and let $\{(x_j,y_j)\}_{j=0}^{N_{\mathrm{space}}-1}$ denote the spatial sampling
points. 
Each transition consists of a four-channel input and a two-channel target:
\begin{equation}
\label{eq:transition-data}
\begin{aligned}
X_{k,j}^{(p)}
&=\bigl(u^{(p)}(x_j,y_j,t_k),v^{(p)}(x_j,y_j,t_k),x_j,y_j\bigr),\\
Y_{k,j}^{(p)}
&=\bigl(u^{(p)}(x_j,y_j,t_{k+1}),v^{(p)}(x_j,y_j,t_{k+1})\bigr).
\end{aligned}
\end{equation}
Here $(u,v)$ is replaced by $(U,V)$ in Experiment~III, and $N_{\mathrm{space}}$ counts spatial points.
The scalar time increment is normalized by the final time of the corresponding experiment:
\begin{equation}
\label{eq:normalized-time-increment}
\delta_k=\frac{t_{k+1}-t_k}{T},\qquad
\Gamma\bigl(X_k^{(p)},\delta_k\bigr)\approx Y_k^{(p)}.
\end{equation}
Thus, the model receives $X_k^{(p)}\in\R^{N_{\mathrm{space}}\times4}$
and the separate scalar $\delta_k$, and returns the next state at the same points in $\R^{N_{\mathrm{space}}\times2}$.
During rollout, the predicted fields replace the field channels of the next input, while the coordinates remain fixed and the prescribed time increment is used.

The code used for the numerical experiments is available at
\url{https://github.com/zawa0073/LENO_public}.

\subsection{Experiment I: Gierer-Meinhardt system on a square}
\label{sec:experiment-i}

We consider the following parabolic system:
\begin{subequations}
\begin{align*}
    &u_t = 0.1 \Delta u + 0.5 \Bigl( \dfrac{u^2}{(1+p u^2)\,v} - u \Bigr),
    \quad \text{in } \Omega \times (0,T), \\
    &v_t = 2.0 \Delta v + 0.5 \bigl( u^2 - 0.9\,v \bigr),
    \quad \text{in } \Omega \times [0,T],\\
    &\frac{\partial u}{\partial \nu}
    = \frac{\partial v}{\partial \nu}
    = 0, \quad \text{on } \partial \Omega \times [0,T], \\
    &u(x,0) = u_0(x),\quad
    v(x,0) = v_0(x),
    \quad \text{in } \overline{\Omega}, \\
    &\frac{\partial u_0}{\partial \nu}
    = \frac{\partial v_0}{\partial \nu}
    = 0,
    \quad \text{on } \partial \Omega. 
\end{align*}
\end{subequations}
where  
\[
x \in \Omega = [0,200]\times[0,200],
\qquad t\in[0,250].
\]
We generate random initial conditions by perturbing a steady equilibrium 
$(u^*,v^*)$ as
\begin{equation}
\label{eq:random-ic}
u_0 = u^* + 0.05\,\xi_u,
\qquad
v_0 = v^* + 0.05\,\xi_v,
\end{equation}
where $\xi_u$ and $\xi_v$ are sampled independently from a standard normal distribution on the spatial grid (i.i.d.\ across grid points).
Thus, the perturbations are additive Gaussian noise with an absolute standard deviation of $0.05$ in each component.
We parametrize the equilibrium by a scalar $s>0$:
\begin{equation*}
u^* = s,
\qquad
v^* = \frac{s^2}{0.9}.
\end{equation*}

Based on \cite{Song2017}, the parameter $p$ is chosen so that
\begin{equation*}
\frac{1}{1+p s^2} = \frac{s}{0.9}.
\end{equation*}
We consider the two parameter values
\begin{equation*}
s \in \{0.785,\ 0.85\}.
\end{equation*}
For each value of $s$, the dataset consists of $90$ training trajectories and $10$ test trajectories, each pairing a random initial condition with its numerical solution. 

We discretize the spatial domain $\Omega$ using a uniform $200\times200$ grid.
In time, we employ a non-uniform discretization that is denser near the initial transient, where the dynamics change most rapidly, and coarser afterward.
Specifically, we use the time grid
$
t \in \{0,1,2,\dots,50,100,150,200,250\}
$
, and store the input coordinates in normalized form:
\begin{equation}
\label{eq:normalization}
(x,y)\mapsto \Bigl(\frac{x}{200},\frac{y}{200}\Bigr)\in[0,1]^2,
\qquad
t\mapsto \frac{t}{250}.
\end{equation}

Both models have depth $4$ and channel width $64$. 
LENO uses $96\times96$ spectral modes, whereas FNO uses $48\times48$ Fourier modes.
The LENO truncation is chosen with reference to the Nyquist--Shannon sampling criterion, balancing predictive accuracy and computational cost.
The two networks have approximately $151\times10^6$ trainable parameters. 
We use Adam with an initial learning rate of $10^{-3}$, halve the learning rate every $50$ epochs, and train for $200$ epochs with a batch size of $8$. 
These optimization settings are identical for FNO and LENO.

\subsection{Experiment II: Gierer-Meinhardt system on a disk}
\label{sec:experiment-ii}

In the second experiment, we compare FNO and LENO on a domain with a curved boundary. 
Following the saturated Gierer-Meinhardt model studied by Moyles and Ward~\cite{MoylesWard2017}, we consider
\begin{equation}
\label{eq:gm-disk}
\begin{aligned}
    u_t &= d_1\Delta u - u
    + \frac{u^2}{v(1+\sigma u^2)},
    &&\text{in }\Omega\times(0,T), \\
    v_t &= d_2\Delta v - v + \frac{u^2}{\sqrt{d_1}},
    &&\text{in }\Omega\times(0,T), \\
    \frac{\partial u}{\partial\nu}
    &= \frac{\partial v}{\partial\nu}=0,
    &&\text{on }\partial\Omega\times[0,T], \\
    u(x,y,0)&=u_0(x,y),\qquad v(x,y,0)=v_0(x,y),
    &&\text{in }\overline{\Omega},
\end{aligned}
\end{equation}
where the spatial domain is the unit disk and the final time is $T=50$:
\[
\Omega=\{(x,y)\in\R^2:x^2+y^2<1\},\qquad t\in[0,50].
\]
Here, $d_1$ and $d_2$ are the diffusion coefficients, and $\sigma$ controls the saturation of the activator production term. 
In terms of the radial coordinate $r=\sqrt{x^2+y^2}$, the boundary conditions are $u_r=v_r=0$ at $r=1$. 
Here, $u_r=\partial u/\partial r$ and $v_r=\partial v/\partial r$. On the circular boundary, the outward unit normal is the radial unit vector, so these derivatives coincide with the outward normal derivatives, without assuming radial symmetry.
Moyles and Ward~\cite[Section~5, p.~629]{MoylesWard2017} describe vanishing radial derivatives at the origin as a compatibility condition in their numerical scheme. In our cell-centered discretization, the origin is not a grid point; its treatment is specified by the radial flux discretization below. Vanishing radial derivatives at the origin follow from smoothness for radially symmetric fields, but are not a general regularity requirement for fields with angular dependence.

We generate the initial conditions by perturbing a spatially homogeneous positive equilibrium $(u^*,v^*)$:
\begin{equation}
\label{eq:gm-disk-ic}
u_0=u^*+5\times10^{-4}\,\xi_u,
\qquad
v_0=v^*+5\times10^{-4}\,\xi_v,
\end{equation}
where $\xi_u$ and $\xi_v$ are sampled independently from the uniform distribution $\mathcal{U}(-1,1)$ at each spatial grid point. 
The noise therefore has zero mean and an absolute amplitude of $5\times10^{-4}$ in each component. 
The equilibrium can be parametrized as
\begin{equation}
\label{eq:gm-disk-equilibrium}
u^*=s,\qquad v^*=\frac{s^2}{\sqrt{d_1}},
\qquad \sigma s^3+s-\sqrt{d_1}=0,\quad s>0.
\end{equation}
These relations follow by setting the reaction terms in \eqref{eq:gm-disk} equal to zero.

We consider two parameter configurations:
\begin{equation}
\label{eq:gm-disk-parameters}
(\sqrt{d_1},d_2,\sigma)
\in\{(0.025,1,25),\ (0.01,0.01,5910)\}.
\end{equation}
The corresponding equilibrium values are listed in Table~\ref{tab:gm-disk-equilibria}. 
For each configuration, FNO and LENO are trained separately using the one-step prediction formulation, and their predictions are advanced by autoregressive rollout. 
This experiment extends the comparison in Experiment~I to a nonrectangular domain with homogeneous Neumann boundary conditions.

\begin{table}[htbp]
\centering
\begin{tabular}{ccc|cc}
\hline
$\sqrt{d_1}$ & $d_2$ & $\sigma$ & $u^*$ & $v^*$ \\
\hline
$0.025$ & $1$    & $25$   & $0.02462662$ & $0.02425881$ \\
$0.01$  & $0.01$ & $5910$ & $0.00750336$ & $0.00563005$ \\
\hline
\end{tabular}
\caption{Parameter configurations and approximate positive homogeneous
equilibria for Experiment~II.}
\label{tab:gm-disk-equilibria}
\end{table}

For each parameter configuration, the teacher dataset contains $100$ trajectories sampled at
\begin{equation}
\label{eq:gm-disk-times}
t_k=k,\qquad k=0,1,\dots,50.
\end{equation}
Each trajectory therefore provides $51$ saved states and $50$ consecutive one-step transitions, with normalized increment $\delta_k=1/50=0.02$.

We sample the disk on a polar grid with $N_r=128$ radial points and $N_\theta=256$ angular points:
\begin{equation}
\label{eq:gm-disk-grid}
\begin{aligned}
r_i&=\frac{i+1/2}{N_r}, &&i=0,\dots,N_r-1,\\
\theta_j&=\frac{2\pi j}{N_\theta}, &&j=0,\dots,N_\theta-1,\\
(x_{ij},y_{ij})&=(r_i\cos\theta_j,r_i\sin\theta_j).
\end{aligned}
\end{equation}
This gives $N_{\mathrm{space}}=128\times256=32768$ sampling points.
The radial coordinates are cell centers, so neither $r=0$ nor $r=1$ is included as a sampling point.
The teacher solver discretizes the radial Laplacian in flux form. It sets the outer radial flux to zero and uses a zero inner-face flux contribution at the origin, where the radial face weight vanishes. This inner-face treatment does not independently enforce pointwise values of $u_r$ and $v_r$ at $r=0$. The angular discretization is periodic.
At each pair of consecutive saved states, the fields and Cartesian coordinates form the input in \eqref{eq:transition-data}, together with $\delta_k$.

For LENO, we use the Neumann eigenfunctions of the Laplacian on the disk.
For a disk of radius $R$, the eigenvalue problem is
\begin{equation}
\label{eq:disk-eigenproblem}
-\Delta\phi_{n,m}=\lambda_{n,m}\phi_{n,m}\quad\text{in }\Omega,
\qquad
\frac{\partial\phi_{n,m}}{\partial r}=0\quad\text{at }r=R.
\end{equation}
Let $J_n$ denote the Bessel function of the first kind of order $n$, and let $\alpha_{n,m}$ be its $m$th positive derivative zero:
\begin{equation}
\label{eq:disk-eigenvalues}
J_n'(\alpha_{n,m})=0,\qquad
\lambda_{n,m}=\left(\frac{\alpha_{n,m}}{R}\right)^2.
\end{equation}
Up to normalization, the real eigenfunctions are
\begin{equation}
\label{eq:disk-bessel-basis}
\begin{aligned}
\phi_{n,m}^{\mathrm{c}}(r,\theta)
&=J_n\!\left(\frac{\alpha_{n,m}r}{R}\right)\cos(n\theta),
&&n\geq0,\ m\geq1,\\
\phi_{n,m}^{\mathrm{s}}(r,\theta)
&=J_n\!\left(\frac{\alpha_{n,m}r}{R}\right)\sin(n\theta),
&&n\geq1,\ m\geq1.
\end{aligned}
\end{equation}
The constant eigenfunction with eigenvalue zero is also included.
In this experiment, $R=1$, and the basis is truncated in ascending order of the eigenvalues.

For FNO, we place the disk inside a surrounding Cartesian square and set the field values outside the disk to zero before applying the FFT.
The spectral convolutions use Fourier modes $\exp\!\bigl(\mathrm{i}(k_xx+k_yy)\bigr)$ on this square. 
Thus, LENO represents the fields in a basis defined on the disk with Neumann boundary conditions, while FNO applies Fourier convolutions to the fields extended by zero to the square.

Both models have depth $4$ and channel width $64$. 
LENO uses $2048$ spectral modes, whereas FNO uses $23\times23$ Fourier modes.
The LENO truncation is selected using the reconstruction error obtained by projecting the teacher fields onto the retained basis functions, balancing predictive accuracy and computational cost. 
The two networks have approximately $33\times10^6$ trainable parameters. 
We use Adam with an initial learning rate of $10^{-3}$, halve the learning rate every $50$ epochs, and train for $200$ epochs with a batch
size of $32$. 
These optimization settings are identical for FNO and LENO.


\subsection{Experiment III: Gray-Scott system on a square}
\label{sec:experiment-iii}

In the third experiment, we consider the Gray-Scott system studied by Pearson~\cite{pearson1993complex} to compare FNO and LENO for a different reaction mechanism. 
The governing equations are
\begin{equation}
\label{eq:gray-scott}
\begin{aligned}
    U_t &= D_u\Delta U-UV^2+F(1-U),
    &&\text{in }\Omega\times(0,T), \\
    V_t &= D_v\Delta V+UV^2-(F+k)V,
    &&\text{in }\Omega\times(0,T), \\
    \frac{\partial U}{\partial\nu}
    &=\frac{\partial V}{\partial\nu}=0,
    &&\text{on }\partial\Omega\times[0,T], \\
    U(x,y,0)&=U_0(x,y),\qquad V(x,y,0)=V_0(x,y),
    &&\text{in }\overline{\Omega},
\end{aligned}
\end{equation}
where
\[
\Omega=[0,2.5]\times[0,2.5],\qquad T=10000,
\]
and the diffusion coefficients are fixed at
\[
D_u=2\times10^{-5},\qquad D_v=1\times10^{-5}.
\]
The parameters $F$ and $k$ determine the feed and removal rates, respectively, and different choices lead to different spatial patterns.
We use the two parameter pairs specified for this experiment:
\begin{equation}
\label{eq:gray-scott-parameters}
(F,k)\in\{(0.0125,0.05),\ (0.024,0.05)\}.
\end{equation}

We discretize the spatial domain using a uniform $256\times256$ grid, which contains $N_{\mathrm{space}}=65536$ sampling points. 
Including the boundary endpoints, the grid is
\[
x_i=\frac{2.5i}{255},\qquad y_j=\frac{2.5j}{255},\qquad
i,j=0,1,\dots,255.
\]
The initial state is $(U,V)=(1,0)$ throughout the domain, except in a central block of $20\times20$ grid points.
Within this block, we perturb the reference values $U=1/2$ and $V=1/4$ by independent multiplicative noise with a relative amplitude of $1\%$.
Writing $P$ for the central block, the initial conditions take the form
\begin{equation}
\label{eq:gray-scott-ic}
\begin{aligned}
U_0(x,y)&=
\begin{cases}
\dfrac12\bigl(1+\eta_U(x,y)\bigr), &(x,y)\in P,\\
1, &(x,y)\notin P,
\end{cases}
\\[0.5em]
V_0(x,y)&=
\begin{cases}
\dfrac14\bigl(1+\eta_V(x,y)\bigr), &(x,y)\in P,\\
0, &(x,y)\notin P,
\end{cases}
\end{aligned}
\end{equation}
where $\eta_U$ and $\eta_V$ are sampled independently from $\mathcal{U}(-0.01,0.01)$ at each grid point in $P$. 
Consequently, the absolute perturbation amplitudes are $0.005$ for $U$ and $0.0025$ for $V$, giving $U_0\in[0.495,0.505]$ and $V_0\in[0.2475,0.2525]$ within the central block.
Outside $P$, the initial state remains exactly $(1,0)$.
The evolution is followed over the time interval $[0,10000]$ under homogeneous Neumann boundary conditions.

For each pair $(F,k)$, we generate a dataset of $100$ trajectories.
The two parameter pairs therefore define two separate datasets.
We save each trajectory on the nonuniform time grid
\begin{equation}
\label{eq:gray-scott-times}
\mathcal{T}=\{0,10,20,\dots,90,100,200,300,\dots,10000\}.
\end{equation}
There are $110$ saved states and $109$ consecutive transitions per trajectory, giving $10900$ transitions per parameter pair before any training/test split. 
Writing $t_0,\dots,t_{109}$ for the ordered elements of $\mathcal{T}$, the normalized time increment is
\begin{equation}
\label{eq:gray-scott-time-increments}
\delta_\ell=\frac{t_{\ell+1}-t_\ell}{10000}
=\begin{cases}
0.001, &0\leq t_\ell<100,\\
0.01, &100\leq t_\ell<10000,
\end{cases}
\quad \ell=0,\dots,108.
\end{equation}
The first transition is from $t=0$ to $t=10$, and the final transition is from $t=9900$ to $t=10000$.
For every transition, the input consists of $(U,V,x,y)$ at all grid points and the normalized increment $\delta_\ell$. 
The target is $(U,V)$ at the next saved time, as in \eqref{eq:transition-data}.

For each pair $(F,k)$, separate FNO and LENO models are trained to predict the next two-component state from the current state. 
As in the preceding experiments, training uses the relative $L^2$ error, and predictions over the full time interval are obtained by autoregressive rollout.

Both models have depth $4$ and channel width $32$. 
LENO uses $128\times128$ spectral modes, whereas FNO uses $64\times64$ Fourier modes.
The LENO truncation is chosen with reference to the Nyquist--Shannon sampling criterion, balancing predictive accuracy and computational cost.
The two networks have approximately $67\times10^6$ trainable parameters each. 
We use Adam with an initial learning rate of $10^{-3}$, halve the learning rate every $50$ epochs, and train for $500$ epochs with a batch size of $16$. 
These optimization settings are identical for FNO and LENO.

\subsection{Numerical results}\label{sec:numerical-results}


In all six tested configurations, LENO achieved lower relative $L^2$ errors on the test trajectories than FNO. 
The difference was small in Experiment~I, whereas it was particularly pronounced in Experiment~II on the disk.
In Experiment~III, LENO also yielded substantially smaller values of the measured normal derivative $\nabla u\cdot n$, while the improvement in relative $L^2$ error depended on the parameter configuration. 
LENO achieved lower mean relative $L^2$ rollout errors than FNO in both configurations of Experiments~I and II, with the largest reduction occurring in Experiment~II at $\sigma=25$. In Experiment~III, FNO achieved lower mean relative $L^2$ rollout errors in both configurations: $35.87\%$ versus $36.72\%$ for $F=0.0125$, and $27.58\%$ versus $35.64\%$ for $F=0.024$, where each pair lists FNO followed by LENO. Nevertheless, LENO yielded smaller $\nabla\hat{u}\cdot n$, in all six configurations. Thus, in Experiment~III, the smaller boundary residuals did not translate into smaller relative $L^2$ rollout errors.
In all comparisons, FNO and LENO used the same initial values for their network parameters.
The same random seed was used to initialize FNO and LENO, with architecture-specific initialization of their spectral layers.

The pronounced improvement on the disk may partly be due to the choice of spectral basis.
On the square, the Neumann Laplacian eigenfunctions are products of trigonometric functions and are relatively close in form to Fourier functions. 
On the disk, they have Bessel-function radial components and trigonometric angular components, reflecting the geometry and boundary conditions more directly. 
This difference in spectral basis may contribute to LENO's stronger performance in Experiment~II.

In Experiment~III, the difference in relative $L^2$ error arose mainly at early times. During this phase, the patterns propagated outward from the center within the interior, while the fields near the boundary remained in the uniform trivial state $(U,V)=(1,0)$. Since the fields were constant near all four sides, their periodic extensions were continuous and their normal derivatives were zero. Thus, the periodicity assumption of FNO was not violated during this phase, leaving no opportunity for LENO's advantage in handling the boundary conditions to manifest itself. At later times, the predictions of both models became uncorrelated with the teacher solution, and the relative $L^2$ error no longer distinguished their performance. Consequently, the overall relative $L^2$ error was determined by the ability to track the early-phase dynamics. In this experiment, LENO lost track of these dynamics earlier, resulting in a smaller error for FNO.

\clearpage

\paragraph{\bf Training and test losses (Figures~\ref{fig:loss-s0785}-\ref{fig:loss-f0024}).}
The blue and orange curves represent the FNO training and test losses,
respectively; the gray and yellow curves represent the LENO training and
test losses, respectively. Both training and test losses are recorded
at every epoch. The loss measures the discrepancy between the one-step
operator prediction and the corresponding reference data.
\begin{figure}[htbp]
  \centering
  \includegraphics[width=\linewidth]{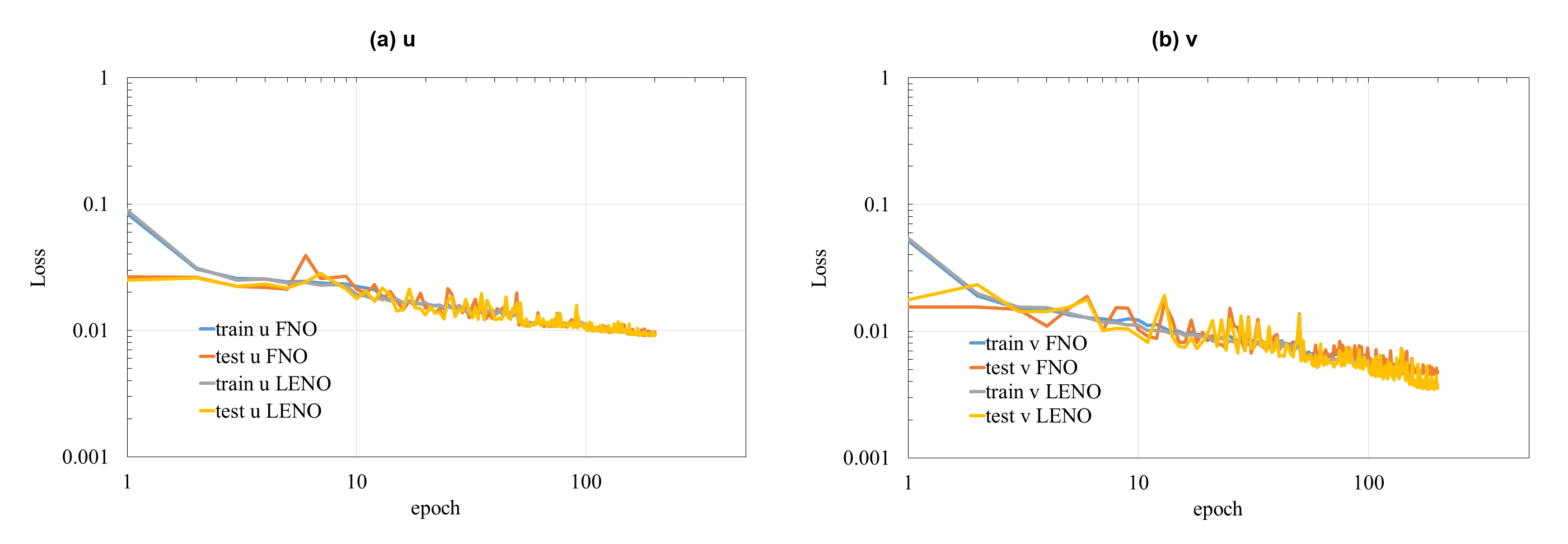}
  \caption{Training and test losses for $u$ and $v$ obtained with
  LENO and FNO in Experiment~I with $s=0.785$.}
  \label{fig:loss-s0785}
\end{figure}

\begin{figure}[htbp]
  \centering
  \includegraphics[width=\linewidth]{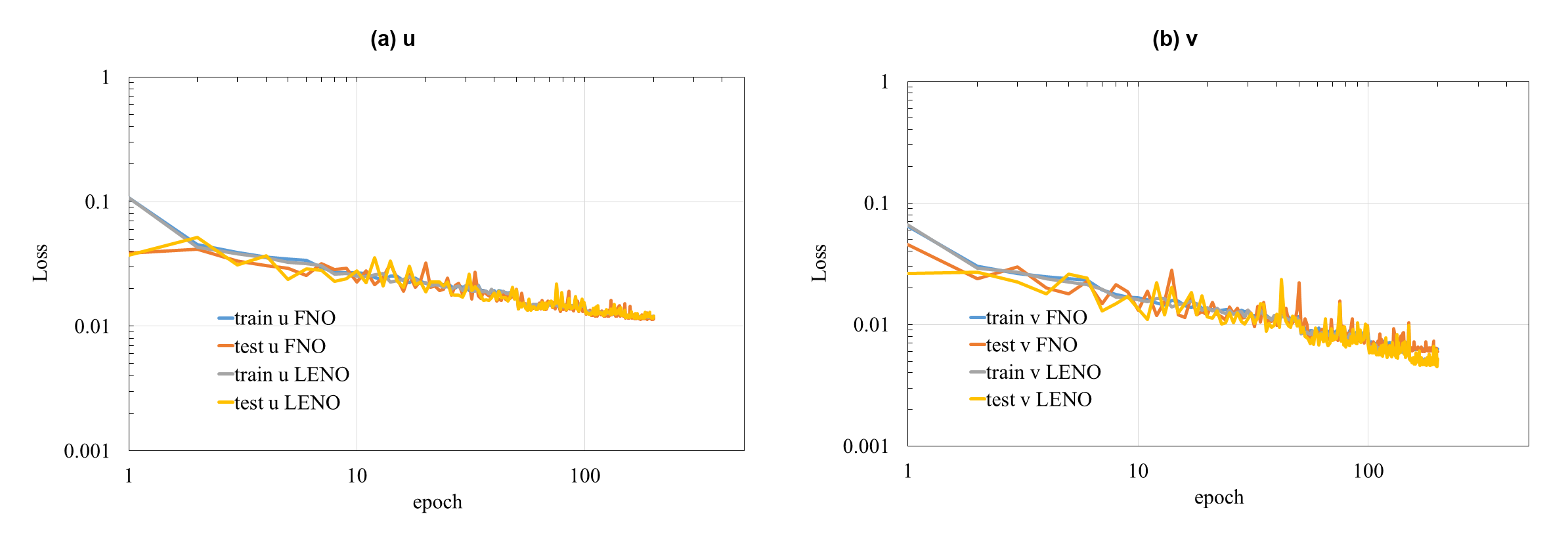}
  \caption{Training and test losses for $u$ and $v$ obtained with
  LENO and FNO in Experiment~I with $s=0.85$.}
  \label{fig:loss-s085}
\end{figure}

\begin{figure}[htbp]
  \centering
  \includegraphics[width=\linewidth]{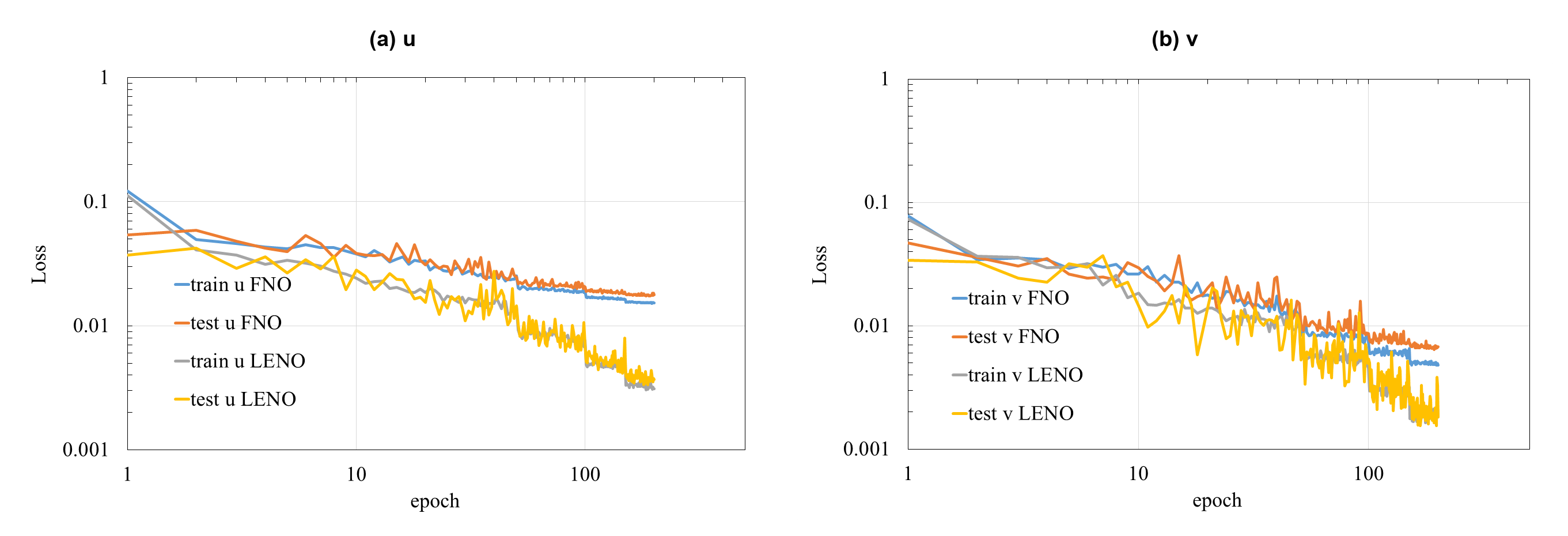}
  \caption{Training and test losses for $u$ and $v$ obtained with
  LENO and FNO in Experiment~II with $\sigma=25$.}
  \label{fig:loss-sigma25}
\end{figure}

\begin{figure}[htbp]
  \centering
  \includegraphics[width=\linewidth]{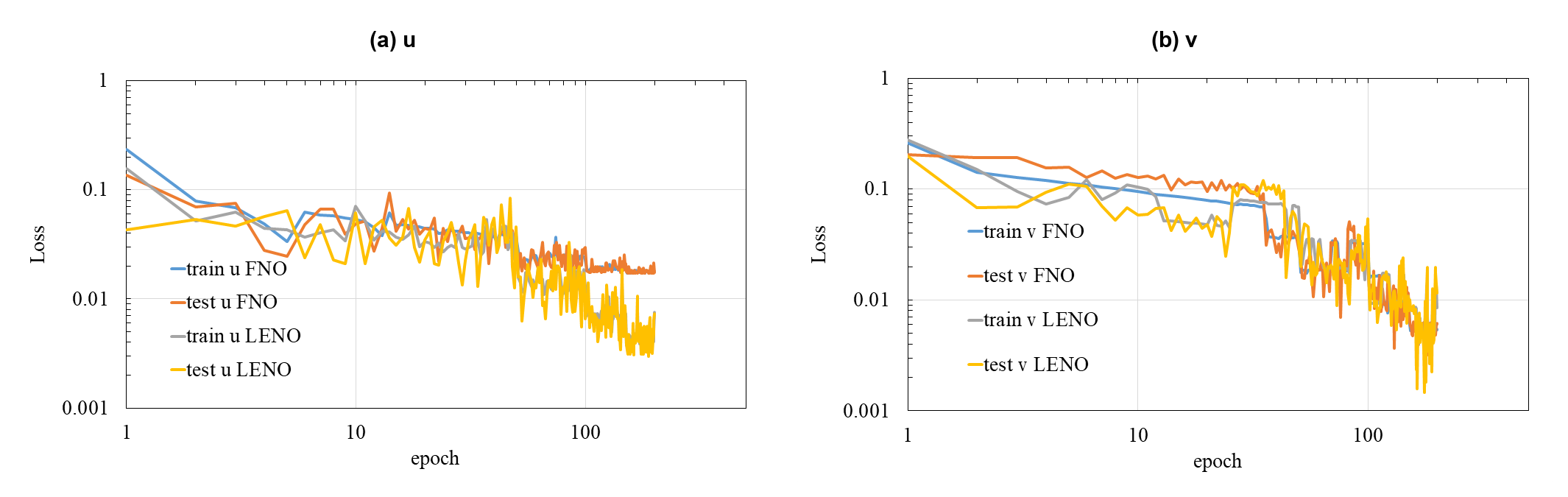}
  \caption{Training and test losses for $u$ and $v$ obtained with
  LENO and FNO in Experiment~II with $\sigma=5910$.}
  \label{fig:loss-sigma5910}
\end{figure}

\begin{figure}[htbp]
  \centering
  \includegraphics[width=\linewidth]{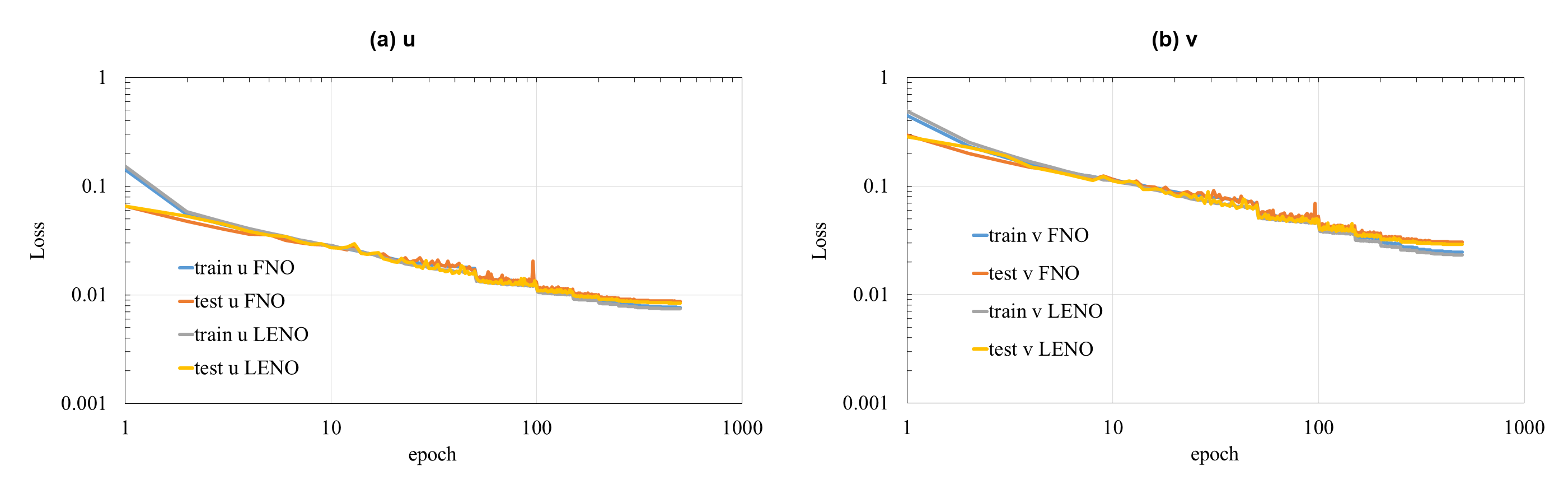}
  \caption{Training and test losses for $U$ and $V$ obtained with
  LENO and FNO in Experiment~III with $F=0.0125$ and $k=0.05$.}
  \label{fig:loss-f00125}
\end{figure}

\begin{figure}[htbp]
  \centering
  \includegraphics[width=\linewidth]{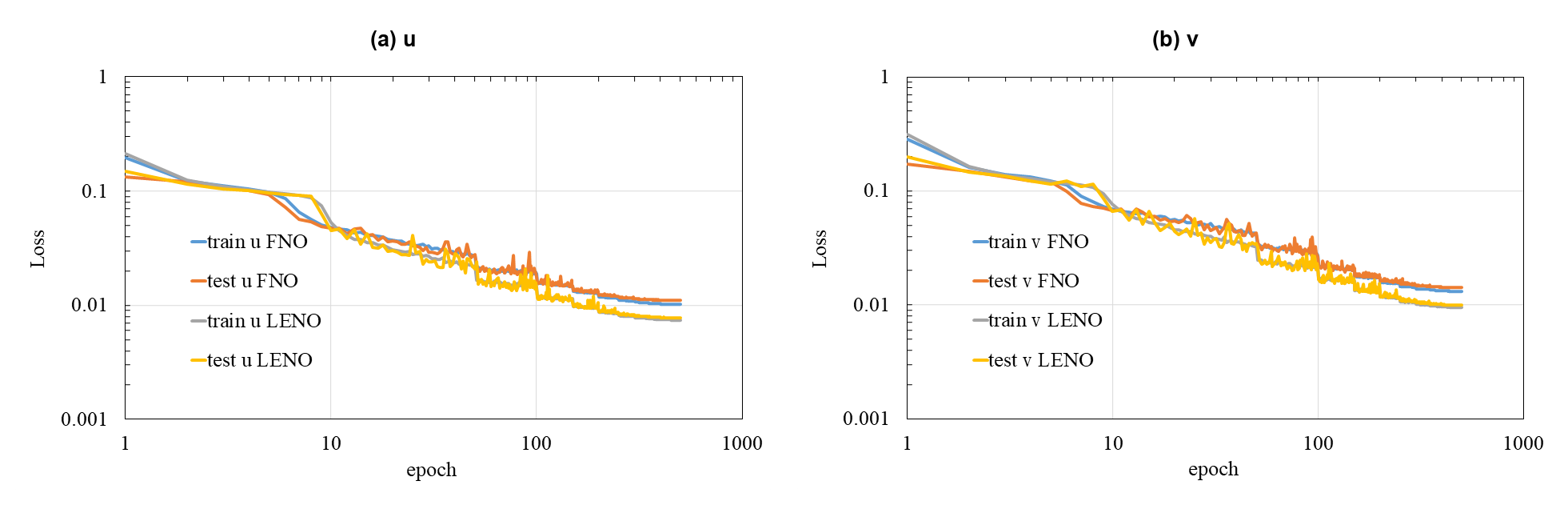}
  \caption{Training and test losses for $U$ and $V$ obtained with
  LENO and FNO in Experiment~III with $F=0.024$ and $k=0.05$.}
  \label{fig:loss-f0024}
\end{figure}

\clearpage

\paragraph{\bf Color maps of predicted solutions (Figures~\ref{fig:colormap-s0785}-\ref{fig:colormap-f0024}).}
The color maps show the full solution fields obtained by repeatedly
applying the trained one-step operator, starting from the initial
conditions in the test dataset.
\begin{figure}[htbp]
  \centering
  \includegraphics[width=\linewidth]{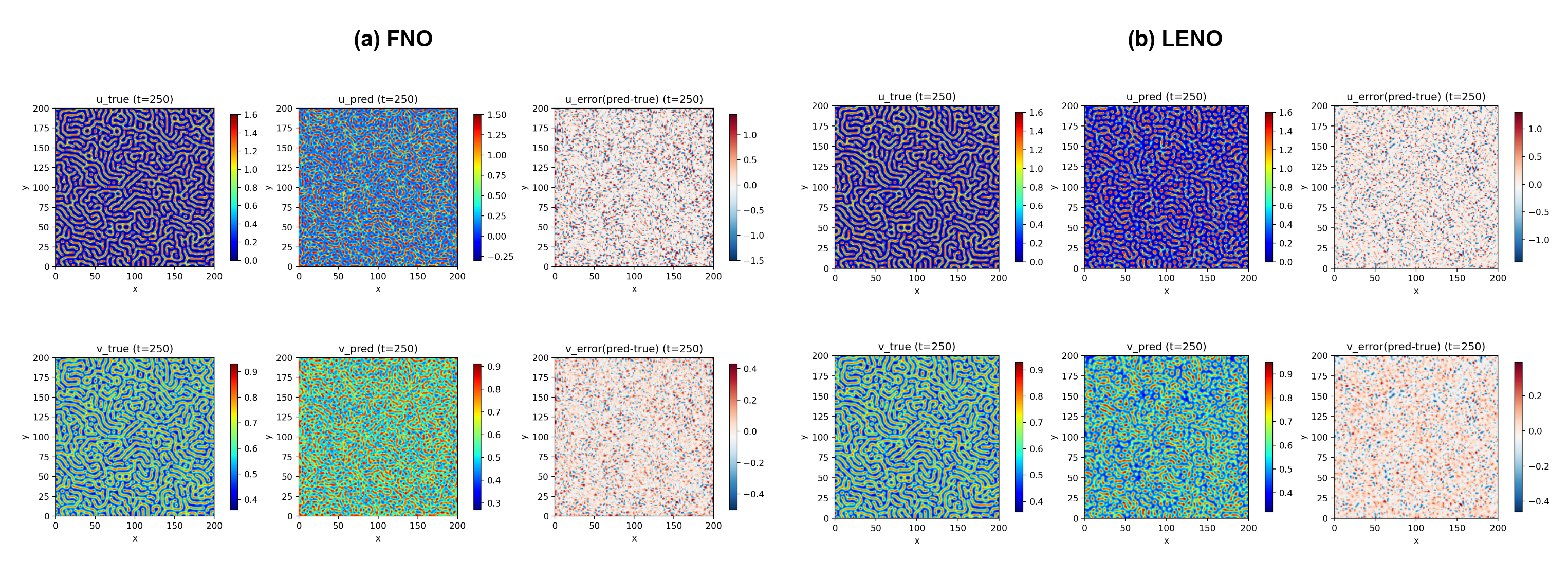}
  \caption{Color maps of the LENO and FNO predictions in
  Experiment~I with $s=0.785$.}
  \label{fig:colormap-s0785}
\end{figure}

\begin{figure}[htbp]
  \centering
  \includegraphics[width=\linewidth]{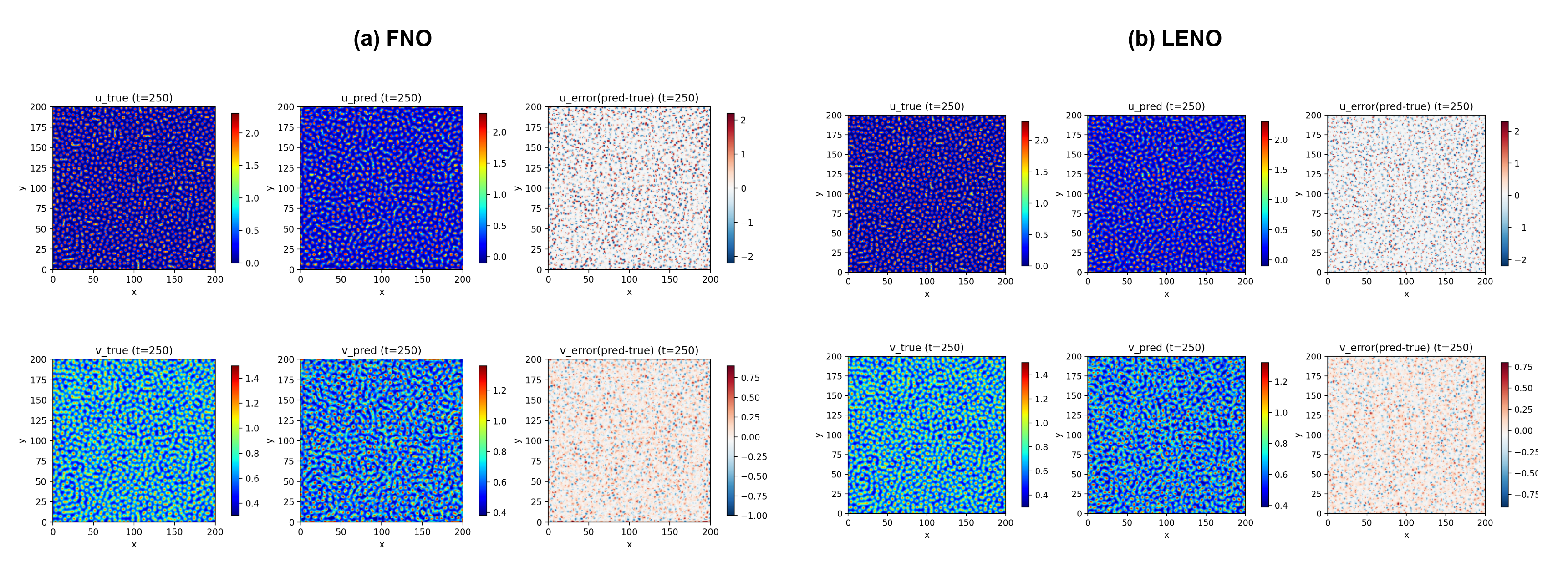}
  \caption{Color maps of the LENO and FNO predictions in
  Experiment~I with $s=0.85$.}
  \label{fig:colormap-s085}
\end{figure}

\begin{figure}[htbp]
  \centering
  \includegraphics[width=\linewidth]{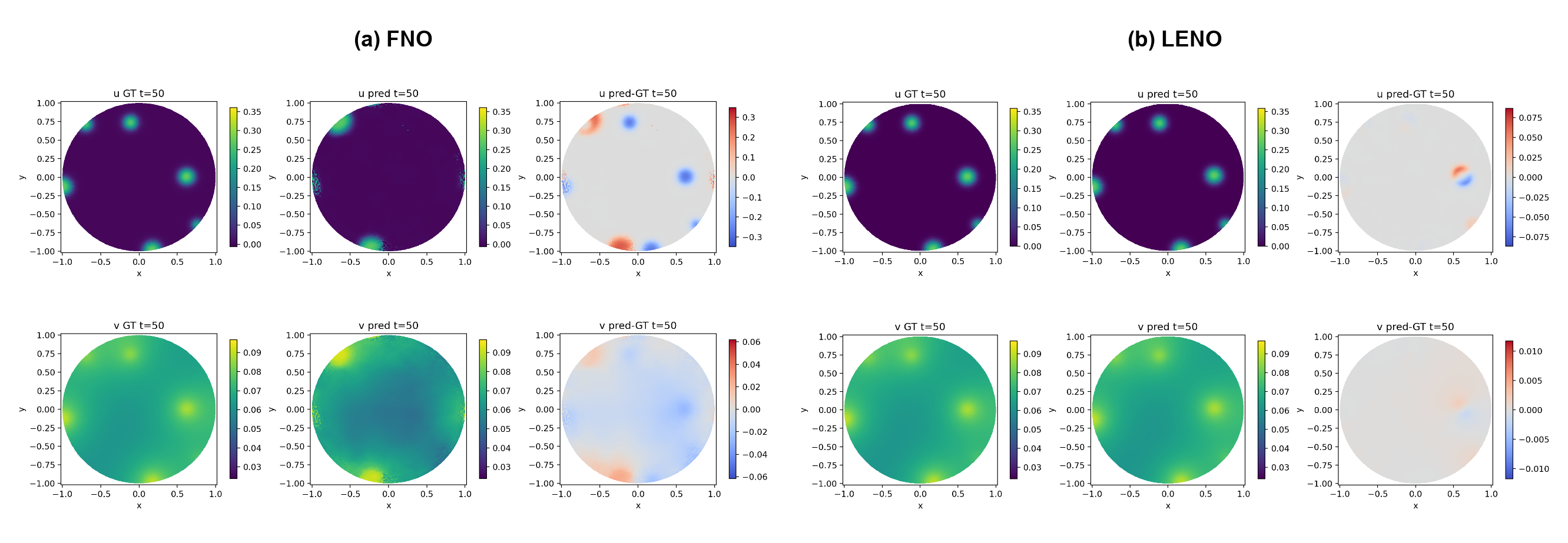}
  \caption{Color maps of the LENO and FNO predictions in
  Experiment~II with $\sigma=25$.}
  \label{fig:colormap-sigma25}
\end{figure}

\begin{figure}[htbp]
  \centering
  \includegraphics[width=\linewidth]{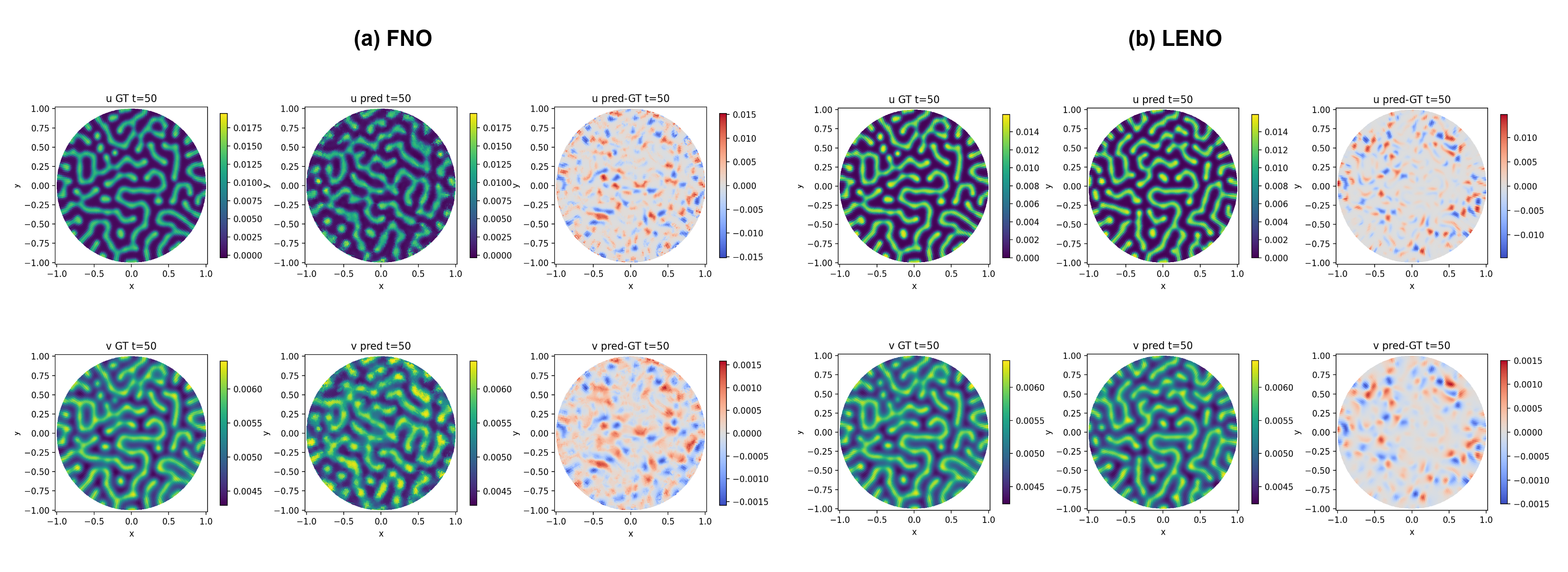}
  \caption{Color maps of the LENO and FNO predictions in
  Experiment~II with $\sigma=5910$.}
  \label{fig:colormap-sigma5910}
\end{figure}

\begin{figure}[htbp]
  \centering
  \includegraphics[width=\linewidth]{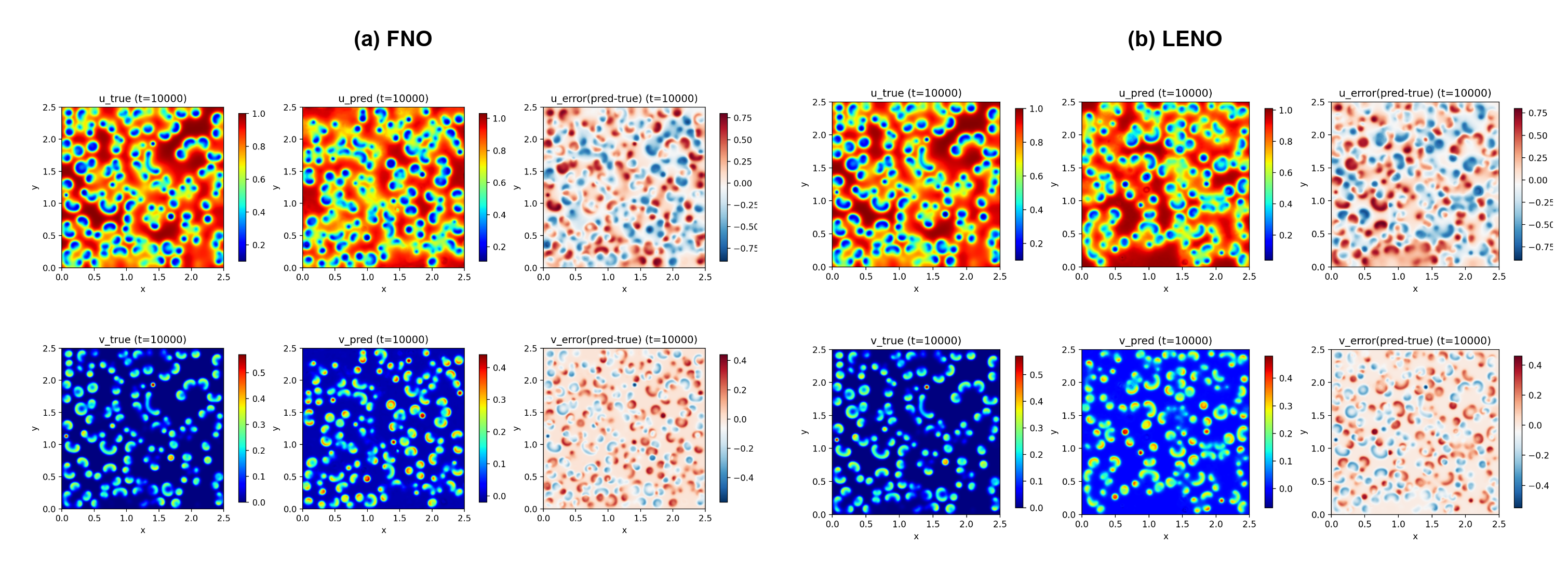}
  \caption{Color maps of the LENO and FNO predictions in
  Experiment~III with $F=0.0125$ and $k=0.05$.}
  \label{fig:colormap-f00125}
\end{figure}

\begin{figure}[htbp]
  \centering
  \includegraphics[width=\linewidth]{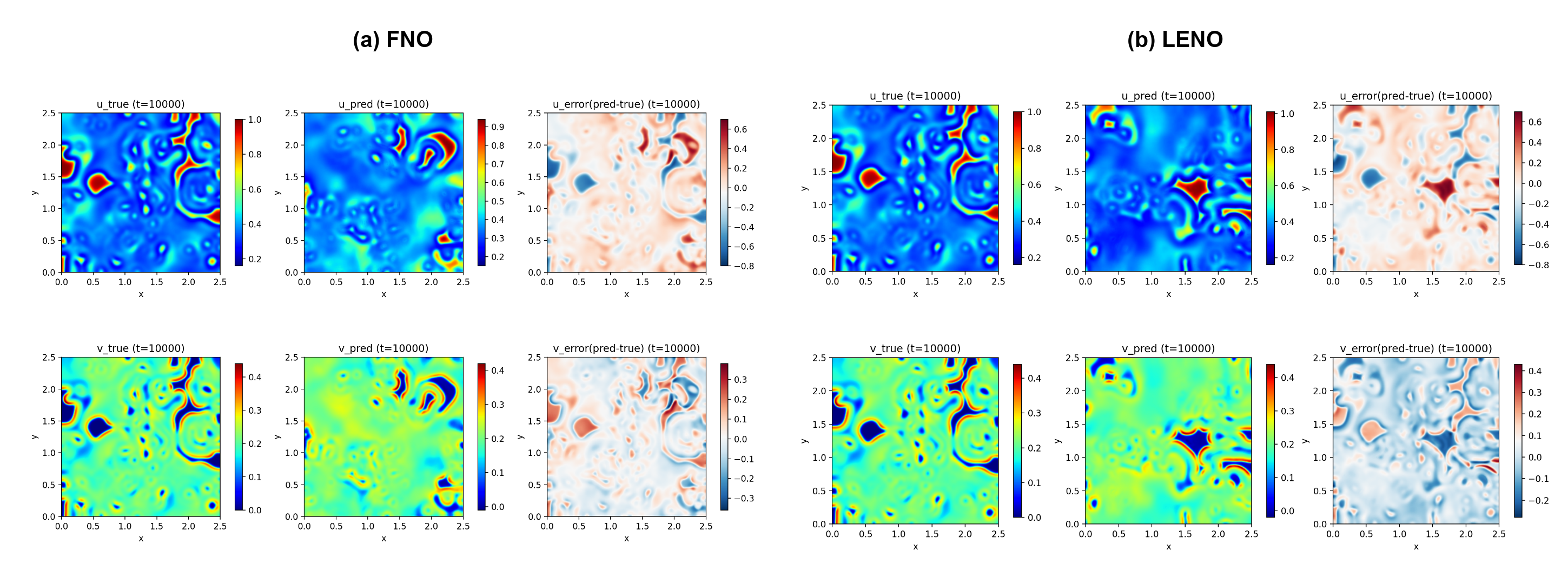}
  \caption{Color maps of the LENO and FNO predictions in
  Experiment~III with $F=0.024$ and $k=0.05$.}
  \label{fig:colormap-f0024}
\end{figure}

\clearpage

\paragraph{\bf Time evolution of the \texorpdfstring{$L^2$}{L2} error (Figures~\ref{fig:timeL2-s0785}-\ref{fig:timeL2-f0024}).}
These plots show the $L^2$ error over the entire spatial domain as a
function of time. The predicted solutions are obtained by repeatedly
applying the trained one-step operator, starting from the initial
conditions in the test dataset.
\begin{figure}[htbp]
  \centering
  \includegraphics[width=\linewidth]{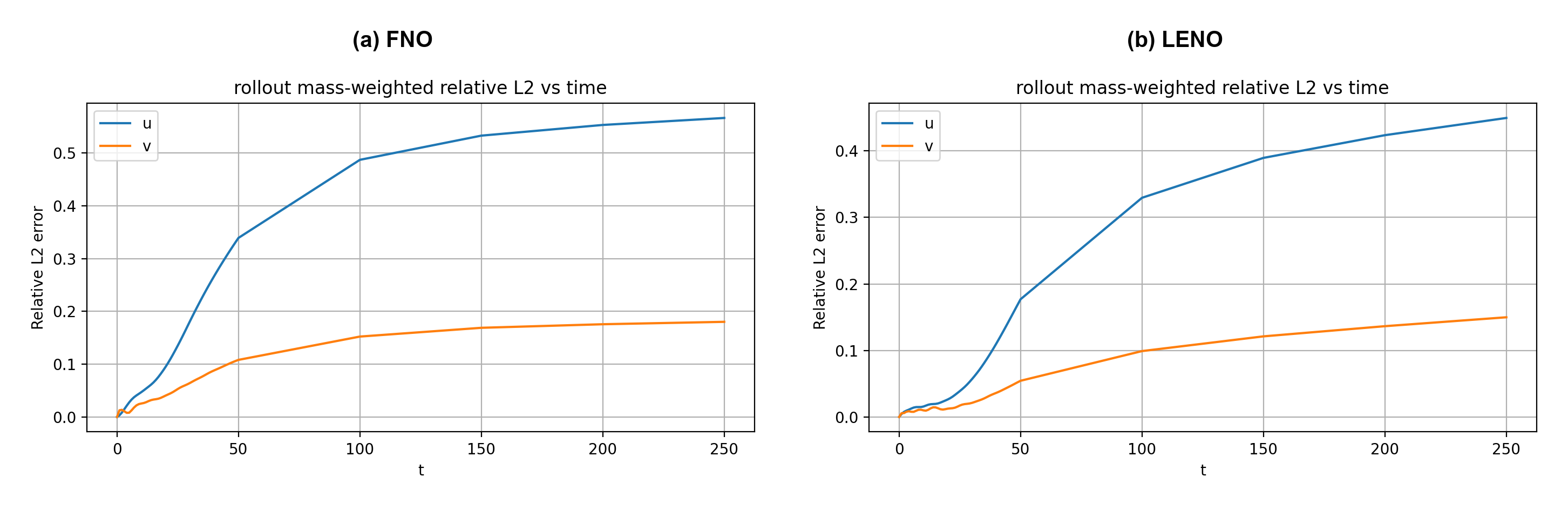}
  \caption{Time evolution of the $L^2$ errors for LENO and FNO in
  Experiment~I with $s=0.785$.}
  \label{fig:timeL2-s0785}
\end{figure}

\begin{figure}[htbp]
  \centering
  \includegraphics[width=\linewidth]{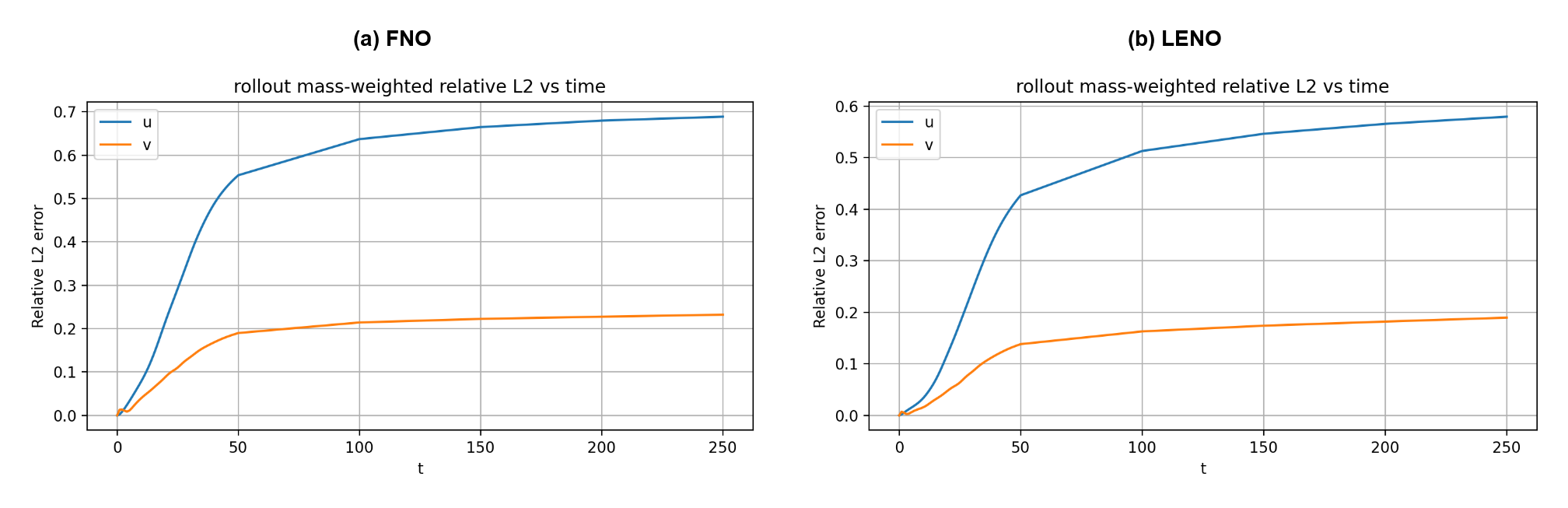}
  \caption{Time evolution of the $L^2$ errors for LENO and FNO in
  Experiment~I with $s=0.85$.}
  \label{fig:timeL2-s085}
\end{figure}

\begin{figure}[htbp]
  \centering
  \includegraphics[width=\linewidth]{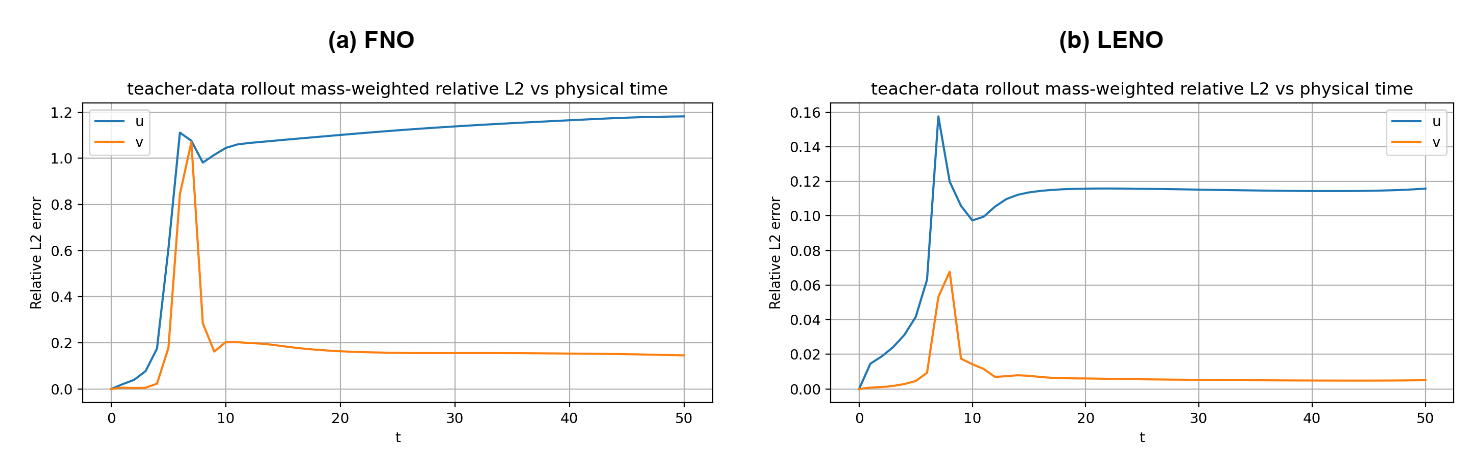}
  \caption{Time evolution of the $L^2$ errors for LENO and FNO in
  Experiment~II with $\sigma=25$.}
  \label{fig:timeL2-sigma25}
\end{figure}

\begin{figure}[htbp]
  \centering
  \includegraphics[width=\linewidth]{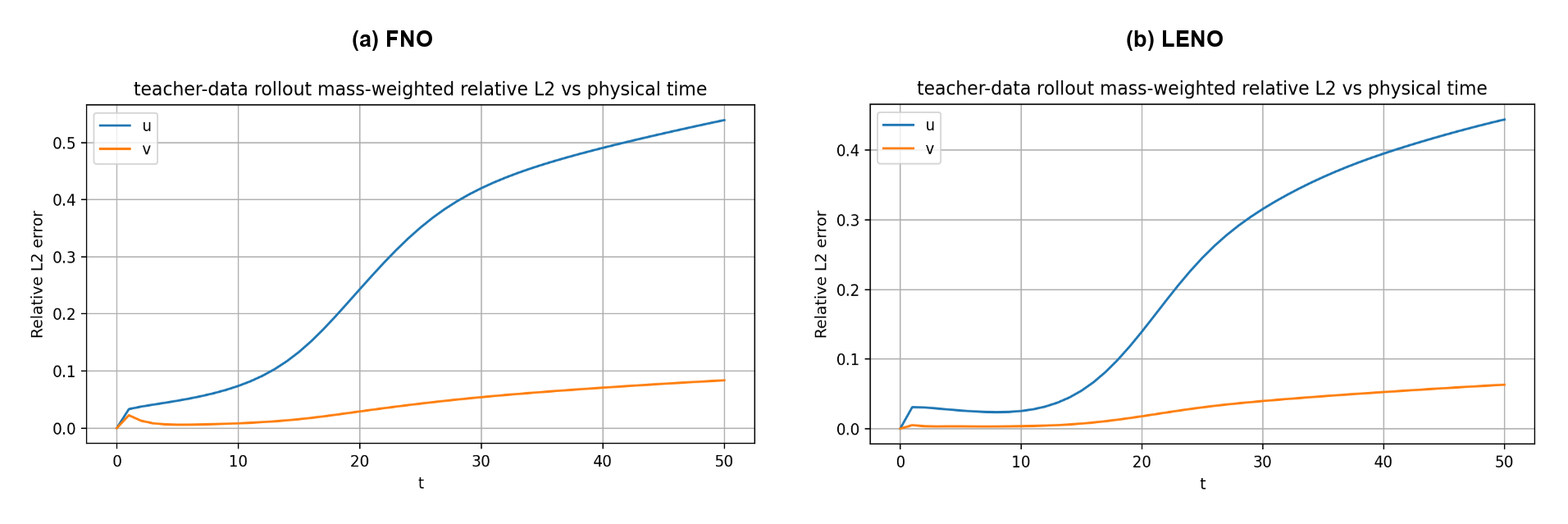}
  \caption{Time evolution of the $L^2$ errors for LENO and FNO in
  Experiment~II with $\sigma=5910$.}
  \label{fig:timeL2-sigma5910}
\end{figure}

\begin{figure}[htbp]
  \centering
  \includegraphics[width=\linewidth]{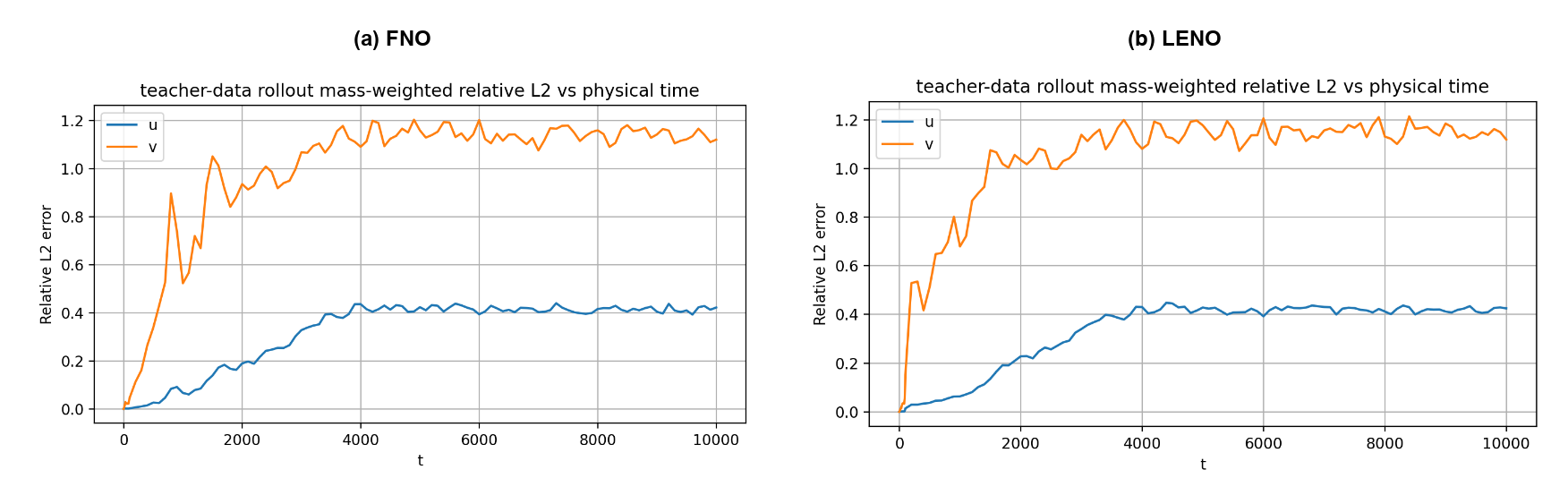}
  \caption{Time evolution of the $L^2$ errors for LENO and FNO in
  Experiment~III with $F=0.0125$ and $k=0.05$.}
  \label{fig:timeL2-f00125}
\end{figure}

\begin{figure}[htbp]
  \centering
  \includegraphics[width=\linewidth]{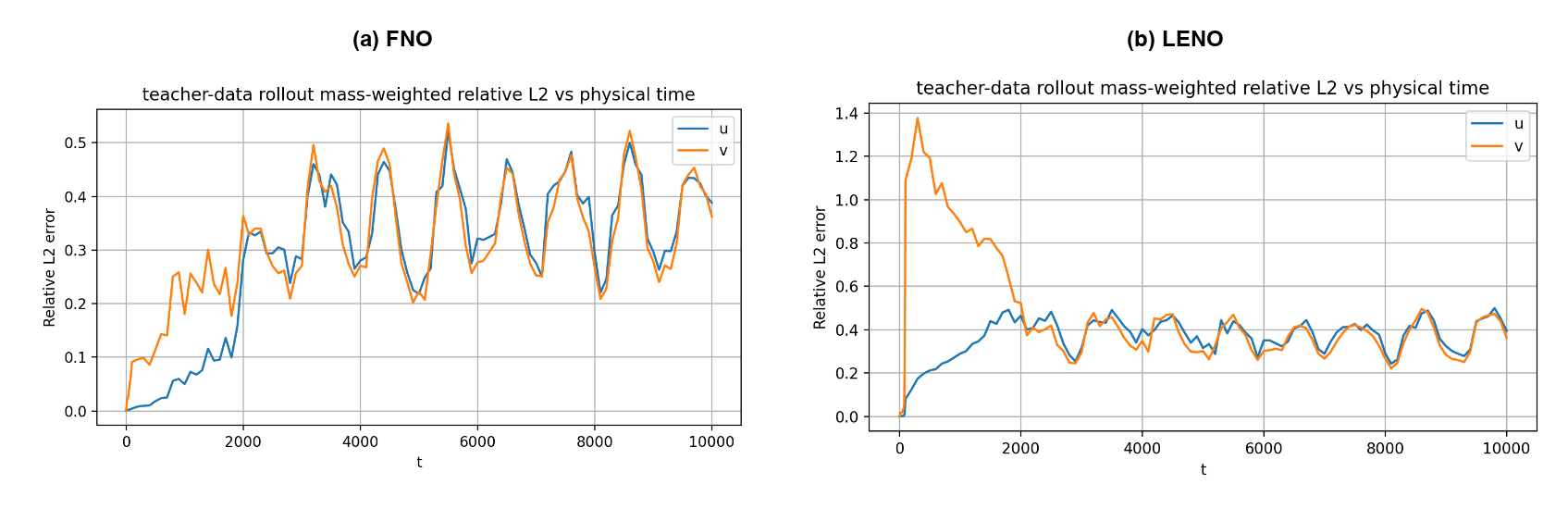}
  \caption{Time evolution of the $L^2$ errors for LENO and FNO in
  Experiment~III with $F=0.024$ and $k=0.05$.}
  \label{fig:timeL2-f0024}
\end{figure}

\clearpage

\begin{table}[htbp]
  \centering
  \caption{Benchmarks on Gierer-Meinhardt (square domain, $s=0.785$, $N=10$).}
  \label{tab:benchmark-gm-square-s0785}
  \small
  \renewcommand{\arraystretch}{1.25}
  \setlength{\tabcolsep}{12pt}
  \begin{tabular}{lcc}
    \hline
    Basis & Relative $L^2$ & $\displaystyle\nabla\hat{u}\cdot n$ \\
    \hline
    Fourier basis & $(37.41 \pm 0.42)\%$ & $34.93 \pm 1.75$ \\
    Laplace basis & $(27.90 \pm 0.31)\%$ & $19.72 \pm 1.10$ \\
    Teacher &  & $8.89 \pm 1.03$ \\
    \hline
  \end{tabular}
\end{table}
The $L^2$ errors in the benchmark tables are evaluated over the full
solution trajectories obtained by repeatedly applying the trained
one-step operator. The normal derivative $\nabla\hat{u}\cdot n$
is evaluated on the same predicted trajectories, where $n$ denotes the
outward unit normal and $\hat{u}$ denotes the predicted first component
($\hat{U}$ in Experiment~III). All entries are reported as the mean
and standard deviation over $N$ test trajectories.

\begin{table}[htbp]
  \centering
  \caption{Benchmarks on Gierer-Meinhardt (square domain, $s=0.85$, $N=10$).}
  \label{tab:benchmark-gm-square-s085}
  \small
  \renewcommand{\arraystretch}{1.25}
  \setlength{\tabcolsep}{12pt}
  \begin{tabular}{lcc}
    \hline
    Basis & Relative $L^2$ & $\displaystyle\nabla\hat{u}\cdot n$ \\
    \hline
    Fourier basis & $(44.79 \pm 0.35)\%$ & $71.24 \pm 3.16$ \\
    Laplace basis & $(37.13 \pm 0.31)\%$ & $47.41 \pm 3.18$ \\
    Teacher &  & $24.54 \pm 3.79$ \\
    \hline
  \end{tabular}
\end{table}
\begin{table}[htbp]
  \centering
  \caption{Benchmarks on Gierer-Meinhardt (circle domain, $\sigma=25$, $N=10$).}
  \label{tab:benchmark-gm-disk-sigma25}
  \small
  \renewcommand{\arraystretch}{1.25}
  \setlength{\tabcolsep}{12pt}
  \begin{tabular}{lcc}
    \hline
    Basis & Relative $L^2$ & $\displaystyle\nabla\hat{u}\cdot n$ \\
    \hline
    Fourier basis & $(59.77 \pm 4.12)\%$ & $20.52 \pm 1.93$ \\
    Laplace basis & $(23.00 \pm 13.70)\%$ & $0.9 \pm 0.19$ \\
    Teacher &  & $0.002 \pm 0.0005$ \\
    \hline
  \end{tabular}
\end{table}
\begin{table}[htbp]
  \centering
  \caption{Benchmarks on Gierer-Meinhardt (circle domain, $\sigma=5910$, $N=10$).}
  \label{tab:benchmark-gm-disk-sigma5910}
  \small
  \renewcommand{\arraystretch}{1.25}
  \setlength{\tabcolsep}{12pt}
  \begin{tabular}{lcc}
    \hline
    Basis & Relative $L^2$ & $\displaystyle\nabla\hat{u}\cdot n$ \\
    \hline
    Fourier basis & $(28.31 \pm 0.66)\%$ & $1.19 \pm 0.071$ \\
    Laplace basis & $(19.63 \pm 1.23)\%$ & $0.05 \pm 0.002$ \\
    Teacher &  & $0.005 \pm 0.0002$ \\
    \hline
  \end{tabular}
\end{table}
\begin{table}[htbp]
  \centering
  \caption{Benchmarks on Gray-Scott (square domain, $F=0.0125$, $N=20$).}
  \label{tab:benchmark-gs-f00125}
  \small
  \renewcommand{\arraystretch}{1.25}
  \setlength{\tabcolsep}{12pt}
  \begin{tabular}{lcc}
    \hline
    Basis & Relative $L^2$ & $\displaystyle\nabla\hat{u}\cdot n$ \\
    \hline
    Fourier basis & $(35.87 \pm 0.14)\%$ & $4.25 \pm 0.18$ \\
    Laplace basis & $(36.72 \pm 0.31)\%$ & $0.25 \pm 0.02$ \\
    Teacher &  & $0.07 \pm 0.002$ \\
    \hline
  \end{tabular}
\end{table}
\begin{table}[htbp]
  \centering
  \caption{Benchmarks on Gray-Scott (square domain, $F=0.024$, $N=20$).}
  \label{tab:benchmark-gs-f0024}
  \small
  \renewcommand{\arraystretch}{1.25}
  \setlength{\tabcolsep}{12pt}
  \begin{tabular}{lcc}
    \hline
    Basis & Relative $L^2$ & $\displaystyle\nabla\hat{u}\cdot n$ \\
    \hline
    Fourier basis & $(27.58 \pm 0.35)\%$ & $4.70 \pm 0.25$ \\
    Laplace basis & $(35.64 \pm 0.32)\%$ & $0.27 \pm 0.03$ \\
    Teacher &  & $0.08 \pm 0.002$ \\
    \hline
  \end{tabular}
\end{table}

\clearpage

\paragraph{Reconstruction errors.}
Tables~\ref{tab:reconstruction-gm-square-s0785}-\ref{tab:reconstruction-gs-f0024}
report the mean and maximum relative $L^2$ reconstruction errors
for the input and target fields in the training data.
The reconstruction settings and displayed precision follow the source tables.

\begin{table}[htbp]
  \centering
  \caption{Reconstruction errors on the training data in Experiment~I ($s=0.785$). Resolution: $201\times201$; $K=96\times96$; tolerance: $1\%$.}
  \label{tab:reconstruction-gm-square-s0785}
  \small
  \renewcommand{\arraystretch}{1.25}
  \setlength{\tabcolsep}{12pt}
  \begin{tabular}{lcc}
    \hline
    Field & Relative $L^2$ mean & Relative $L^2$ max \\
    \hline
    $u(x,y,t)$ & $2.08\%$ & $15.32\%$ \\
    $v(x,y,t)$ & $0.39\%$ & $2.82\%$ \\
    $u(x,y,t+\Delta t)$ & $2.36\%$ & $15.32\%$ \\
    $v(x,y,t+\Delta t)$ & $0.44\%$ & $2.82\%$ \\
    \hline

  \end{tabular}
\end{table}

\begin{table}[htbp]
  \centering
  \caption{Reconstruction errors on the training data in Experiment~I ($s=0.85$). Resolution: $201\times201$; $K=96\times96$; tolerance: $1\%$.}
  \label{tab:reconstruction-gm-square-s085}
  \small
  \renewcommand{\arraystretch}{1.25}
  \setlength{\tabcolsep}{12pt}
  \begin{tabular}{lcc}
    \hline
    Field & Relative $L^2$ mean & Relative $L^2$ max \\
    \hline
    $u(x,y,t)$ & $8.44\%$ & $26.62\%$ \\
    $v(x,y,t)$ & $1.84\%$ & $5.57\%$ \\
    $u(x,y,t+\Delta t)$ & $8.93\%$ & $26.77\%$ \\
    $v(x,y,t+\Delta t)$ & $1.94\%$ & $5.59\%$ \\
    \hline
  \end{tabular}
\end{table}

\begin{table}[htbp]
  \centering
  \caption{Reconstruction errors on the training data in Experiment~II ($\sigma=25$). Resolution: $128\times256$; $K=2048$; tolerance: $1\%$.}
  \label{tab:reconstruction-gm-disk-sigma25}
  \small
  \renewcommand{\arraystretch}{1.25}
  \setlength{\tabcolsep}{12pt}
  \begin{tabular}{lcc}
    \hline
    Field & Relative $L^2$ mean & Relative $L^2$ max \\
    \hline
    $u(x,y,t)$ & $0.18\%$ & $1.13\%$ \\
    $v(x,y,t)$ & $0.02\%$ & $1.15\%$ \\
    $u(x,y,t+\Delta t)$ & $0.16\%$ & $0.26\%$ \\
    $v(x,y,t+\Delta t)$ & $0.0002\%$ & $0.0007\%$ \\
    \hline
  \end{tabular}
\end{table}

\begin{table}[htbp]
  \centering
  \caption{Reconstruction errors on the training data in Experiment~II ($\sigma=5910$). Resolution: $128\times256$; $K=2048$; tolerance: $1\%$.}
  \label{tab:reconstruction-gm-disk-sigma5910}
  \small
  \renewcommand{\arraystretch}{1.25}
  \setlength{\tabcolsep}{12pt}
  \begin{tabular}{lcc}
    \hline
    Field & Relative $L^2$ mean & Relative $L^2$ max \\
    \hline
    $u(x,y,t)$ & $3.10\%$ & $6.25\%$ \\
    $v(x,y,t)$ & $0.19\%$ & $4.94\%$ \\
    $u(x,y,t+\Delta t)$ & $3.15\%$ & $6.27\%$ \\
    $v(x,y,t+\Delta t)$ & $0.10\%$ & $0.20\%$ \\
    \hline
  \end{tabular}
\end{table}

\begin{table}[htbp]
  \centering
  \caption{Reconstruction errors on the training data in Experiment~III ($F=0.0125$, $k=0.05$). Resolution: 
  $256\times256$; 
  $K=128\times128$
  ; tolerance: $1\%$.} 
  \label{tab:reconstruction-gs-f00125}
  \small
  \renewcommand{\arraystretch}{1.25}
  \setlength{\tabcolsep}{12pt}
  \begin{tabular}{lcc}
    \hline
    Field & Relative $L^2$ mean & Relative $L^2$ max \\
    \hline
    $U(x,y,t)$ & $0.05\%$ & $0.48\%$ \\
    $V(x,y,t)$ & $0.87\%$ & $12.41\%$ \\
    $U(x,y,t+\Delta t)$ & $0.05\%$ & $0.08\%$ \\
    $V(x,y,t+\Delta t)$ & $0.76\%$ & $2.67\%$ \\
    \hline
  \end{tabular}
\end{table}

\begin{table}[htbp]
  \centering
  \caption{Reconstruction errors on the training data in Experiment~III ($F=0.024$, $k=0.05$). Resolution: $256\times256$; 
  $K=128\times128$
  ; tolerance: $1\%$.} 
  \label{tab:reconstruction-gs-f0024}
  \small
  \renewcommand{\arraystretch}{1.25}
  \setlength{\tabcolsep}{12pt}
  \begin{tabular}{lcc}
    \hline
    Field & Relative $L^2$ mean & Relative $L^2$ max \\
    \hline
    $U(x,y,t)$ & $0.06\%$ & $0.48\%$ \\
    $V(x,y,t)$ & $0.54\%$ & $12.41\%$ \\
    $U(x,y,t+\Delta t)$ & $0.05\%$ & $0.14\%$ \\
    $V(x,y,t+\Delta t)$ & $0.43\%$ & $2.45\%$ \\
    \hline
  \end{tabular}
\end{table}

\section{Conclusion}

In this work, we established a quantitative approximation theory for Laplacian-based neural operators applied to nonlinear reaction-diffusion systems. 
By exploiting the spectral representation of the diffusion semigroup and the mild formulation of the PDE, we constructed a neural operator that approximates the associated Picard iteration and obtained explicit bounds on its approximation error and spectral rank. 
Our numerical experiments further showed that using Laplacian eigenfunctions can improve prediction accuracy over a Fourier-based model, with the most pronounced improvement observed on the disk domain.

The analysis suggests that eigenfunction expansions provide a natural framework for incorporating PDE geometry and boundary conditions into neural operator approximation theory. 
An important direction for future work is to extend this spectral approach to other classes of PDEs and to identify the corresponding quantitative approximation bounds. 
Our experiments used an autoregressive one-step model for computational efficiency, rather than the time-integrated Picard architecture constructed in the proof. 
Implementing and evaluating that architecture would help determine whether its theoretical guarantees can be realized in a practical model.

\appendix

\end{sloppypar}

\bibliographystyle{plain} 
\bibliography{ref}

@article{pearson1993complex,
	author = {Pearson, John E.},
	doi = {10.1126/science.261.5118.189},
	journal = {Science},
	number = {5118},
	pages = {189-192},
	title = {Complex Patterns in a Simple System},
	volume = {261},
	year = {1993}}

@article{nagumo1962active,
	author = {Nagumo, Jin-ichi and Arimoto, Suguru and Yoshizawa, Shuji},
	doi = {10.1109/JRPROC.1962.288235},
	journal = {Proceedings of the IRE},
	number = {10},
	pages = {2061-2070},
	title = {An Active Pulse Transmission Line Simulating Nerve Axon},
	volume = {50},
	year = {1962}}

@article{fitzhugh1961impulses,
	author = {FitzHugh, Richard},
	doi = {10.1016/S0006-3495(61)86902-6},
	journal = {Biophysical Journal},
	number = {6},
	pages = {445-466},
	title = {Impulses and Physiological States in Theoretical Models of Nerve Membrane},
	volume = {1},
	year = {1961}}

@article{gray1984autocatalytic,
	author = {Gray, Peter and Scott, Stephen K.},
	doi = {10.1016/0009-2509(84)87017-7},
	journal = {Chemical Engineering Science},
	number = {6},
	pages = {1087-1097},
	title = {Autocatalytic Reactions in the Isothermal, Continuous Stirred Tank Reactor: Oscillations and Instabilities in the System A + 2B $\rightarrow$ 3B, B $\rightarrow$ C},
	volume = {39},
	year = {1984}}

@article{kolmogorov1937study,
	author = {Kolmogorov, Andrei N. and Petrovskii, Ivan G. and Piskunov, Nikolai S.},
	journal = {Moscow University Mathematics Bulletin},
	note = {English translation of the original Russian paper},
	number = {6},
	pages = {1-25},
	title = {A Study of the Diffusion Equation with Increase in the Amount of Substance, and Its Application to a Biological Problem},
	volume = {1},
	year = {1937}}

@article{fisher1937wave,
	author = {Fisher, Ronald A.},
	doi = {10.1111/j.1469-1809.1937.tb02153.x},
	journal = {Annals of Eugenics},
	number = {4},
	pages = {355-369},
	title = {The Wave of Advance of Advantageous Genes},
	volume = {7},
	year = {1937}}

@article{Song2017,
	author = {Yongli Song and Rui Yang and Guiquan Sun},
	doi = {https://doi.org/10.1016/j.apm.2017.01.081},
	issn = {0307-904X},
	journal = {Applied Mathematical Modelling},
	pages = {476-491},
	title = {Pattern dynamics in a Gierer-Meinhardt model with a saturating term},
	volume = {46},
	year = {2017}}

@article{yarotsky2017error,
	author = {Yarotsky, Dmitry},
	journal = {Neural networks},
	pages = {103-114},
	publisher = {Elsevier},
	title = {Error bounds for approximations with deep ReLU networks},
	volume = {94},
	year = {2017}}

@article{Furuya-Wang25,
	author = {Takashi Furuya and Jenn-Nan Wang},
	journal = {preprint},
	title = {Quantitative approximation theorems for scattering neural operators},
	year = {2025}}

@book{rothe2006global,
	author = {Rothe, Franz},
	publisher = {Springer},
	title = {Global solutions of reaction-diffusion systems},
	volume = {1072},
	year = {2006}}

@article{gierer1972theory,
	author = {Gierer, Alfred and Meinhardt, Hans},
	journal = {Kybernetik},
	pages = {30-39},
	publisher = {Springer},
	title = {A theory of biological pattern formation},
	volume = {12},
	year = {1972}}

@article{lu2021learning,
	author = {Lu, Lu and Jin, Pengzhan and Pang, Guofei and Zhang, Zhongqiang and Karniadakis, George Em},
	journal = {Nature machine intelligence},
	number = {3},
	pages = {218-229},
	publisher = {Nature Publishing Group UK London},
	title = {Learning nonlinear operators via DeepONet based on the universal approximation theorem of operators},
	volume = {3},
	year = {2021}}

@article{lanthaler2023operator,
	author = {Lanthaler, Samuel},
	journal = {Journal of Machine Learning Research},
	number = {318},
	pages = {1-67},
	title = {Operator learning with PCA-Net: upper and lower complexity bounds},
	volume = {24},
	year = {2023}}

@article{chen1995approximation,
	author = {Chen, Tianping and Chen, Hong},
	journal = {IEEE Transactions on Neural Networks},
	number = {4},
	pages = {904-910},
	publisher = {IEEE},
	title = {Approximation capability to functions of several variables, nonlinear functionals, and operators by radial basis function neural networks},
	volume = {6},
	year = {1995}}

@article{cybenko1989approximation,
	author = {Cybenko, George},
	journal = {Mathematics of control, signals and systems},
	number = {4},
	pages = {303-314},
	publisher = {Springer},
	title = {Approximation by superpositions of a sigmoidal function},
	volume = {2},
	year = {1989}}

@article{castro2024calderon,
	author = {Castro, Javier and Mu{\~n}oz, Claudio and Valenzuela, Nicol{\'a}s},
	journal = {Vietnam Journal of Mathematics},
	pages = {1-32},
	publisher = {Springer},
	title = {The Calder{\'o}n's problem via DeepONets},
	year = {2024}}

@article{marcati2023exponential,
	author = {Marcati, Carlo and Schwab, Christoph},
	journal = {SIAM Journal on Numerical Analysis},
	number = {3},
	pages = {1513-1545},
	publisher = {SIAM},
	title = {Exponential convergence of deep operator networks for elliptic partial differential equations},
	volume = {61},
	year = {2023}}

@article{kovachki2023neural,
	author = {Kovachki, Nikola B and Li, Zongyi and Liu, Burigede and Azizzadenesheli, Kamyar and Bhattacharya, Kaushik and Stuart, Andrew M and Anandkumar, Anima},
	journal = {J. Mach. Learn. Res.},
	number = {89},
	pages = {1-97},
	title = {Neural Operator: Learning Maps Between Function Spaces With Applications to {PDEs}.},
	volume = {24},
	year = {2023}}

@article{lanthaler2023nonlocal,
	author = {Lanthaler, Samuel and Li, Zongyi and Stuart, Andrew M},
	journal = {arXiv preprint arXiv:2304.13221},
	title = {The Nonlocal Neural Operator: Universal Approximation},
	year = {2023}}

@article{lanthaler2022error,
	author = {Lanthaler, Samuel and Mishra, Siddhartha and Karniadakis, George E},
	journal = {Transactions of Mathematics and Its Applications},
	number = {1},
	pages = {tnac001},
	publisher = {Oxford University Press},
	title = {Error estimates for deeponets: A deep learning framework in infinite dimensions},
	volume = {6},
	year = {2022}}

@article{lanthaler2023curse,
	author = {Lanthaler, Samuel and Stuart, Andrew M},
	journal = {arXiv preprint arXiv:2306.15924},
	title = {The curse of dimensionality in operator learning},
	year = {2023}}

@article{li2020fourier,
	author = {Li, Zongyi and Kovachki, Nikola and Azizzadenesheli, Kamyar and Liu, Burigede and Bhattacharya, Kaushik and Stuart, Andrew and Anandkumar, Anima},
	journal = {arXiv preprint arXiv:2010.08895},
	title = {Fourier neural operator for parametric partial differential equations},
	year = {2020}}

@article{tripura2022wavelet,
	author = {Tapas Tripura and Souvik Chakraborty},
	issn = {0045-7825},
	journal = {Computer Methods in Applied Mechanics and Engineering},
	pages = {115783},
	title = {Wavelet Neural Operator for solving parametric partial differential equations in computational mechanics problems},
	volume = {404},
	year = {2023}}

@article{chen2023learning,
	author = {Chen, Gengxiang and Liu, Xu and Meng, Qinglu and Chen, Lu and Liu, Changqing and Li, Yingguang},
	journal = {arXiv preprint arXiv:2302.08166},
	title = {Learning neural operators on {R}iemannian manifolds},
	year = {2023}}

@article{kovachki2021universal,
	author = {Kovachki, Nikola and Lanthaler, Samuel and Mishra, Siddhartha},
	journal = {Journal of Machine Learning Research},
	number = {290},
	pages = {1-76},
	title = {On universal approximation and error bounds for {F}ourier neural operators},
	volume = {22},
	year = {2021}}

@article{calvello2024continuum,
	author = {Calvello, Edoardo and Kovachki, Nikola B and Levine, Matthew E and Stuart, Andrew M},
	journal = {arXiv preprint arXiv:2406.06486},
	title = {Continuum Attention for Neural Operators},
	year = {2024}}

@article{deng2021convergence,
	author = {Beichuan Deng and Yeonjong Shin and Lu Lu and Zhongqiang Zhang and George Em Karniadakis},
	issn = {0893-6080},
	journal = {Neural Networks},
	pages = {411-426},
	title = {Approximation rates of DeepONets for learning operators arising from advection--diffusion equations},
	volume = {153},
	year = {2022}}

@article{li2020neural,
	author = {Li, Zongyi and Kovachki, Nikola and Azizzadenesheli, Kamyar and Liu, Burigede and Bhattacharya, Kaushik and Stuart, Andrew and Anandkumar, Anima},
	journal = {arXiv preprint arXiv:2003.03485},
	title = {Neural operator: Graph kernel network for partial differential equations},
	year = {2020}}

@inproceedings{bonev2023spherical,
	author = {Bonev, Boris and Kurth, Thorsten and Hundt, Christian and Pathak, Jaideep and Baust, Maximilian and Kashinath, Karthik and Anandkumar, Anima},
	booktitle = {International conference on machine learning},
	organization = {PMLR},
	pages = {2806-2823},
	title = {Spherical fourier neural operators: Learning stable dynamics on the sphere},
	year = {2023}}

@article{gupta2021multiwavelet,
	author = {Gupta, Gaurav and Xiao, Xiongye and Bogdan, Paul},
	journal = {Advances in neural information processing systems},
	pages = {24048-24062},
	title = {Multiwavelet-based operator learning for differential equations},
	volume = {34},
	year = {2021}}

@article{chen2023deep,
	author = {Ke Chen and Chunmei Wang and Haizhao Yang},
	issn = {2835-8856},
	journal = {Transactions on Machine Learning Research},
	title = {Deep Operator Learning Lessens the Curse of Dimensionality for {PDE}s},
	year = {2023}}

@article{furuya2024quantitative,
	author = {Furuya, Takashi and Taniguchi, Koichi and Okuda, Satoshi},
	journal = {arXiv preprint arXiv:2410.02151},
	title = {Quantitative Approximation for Neural Operators in Nonlinear Parabolic Equations},
	year = {2024}}

@article{hao2025laplacian,
  title={Laplacian eigenfunction-based neural operator for learning nonlinear partial differential equations},
  author={Hao, Wenrui and Wang, Jindong},
  journal={arXiv preprint arXiv:2502.05571},
  year={2025}
}

@article{MoylesWard2017,
  title={Existence, stability, and dynamics of ring and near-ring solutions to the saturated Gierer--Meinhardt model in the semistrong regime},
  author={Moyles, Iain R and Ward, Michael J},
  journal={SIAM Journal on Applied Dynamical Systems},
  volume={16},
  number={1},
  pages={597--639},
  year={2017},
  publisher={SIAM}
}

\end{document}